\documentclass[journal]{IEEEtran}
\usepackage{amssymb}
\usepackage[dvips]{graphicx}
\usepackage{algorithm}
\usepackage{algorithmic}
\usepackage{amsmath}
\usepackage{amsthm}
\usepackage{url}
\usepackage{color}
\usepackage{booktabs}
\usepackage{times}
\usepackage{helvet}
\usepackage{courier}
\usepackage{threeparttable}
\usepackage{array}
\usepackage{multirow}

\ifCLASSINFOpdf

\else

\fi

\begin{document}

\title{Supervised Categorical Metric Learning with Schatten $p$-Norms}


\author{Xuhui~Fan,~Eric~Gaussier
\thanks{Xuhui Fan is with School of Mathematics \& Statistics, University of New South Wales, Australia.: (email: xhfan.ml@gmail.com).}
\thanks{Eric Gaussier is with Laboratory of Informatics of Grenoble, University of Joseph Fourier, France: (email: eric.gaussier@imag.fr).}
}

\markboth{IEEE transactions on Cybernetics,~Vol.~XX, No.~X, XXX~XXXX}%
{Shell \MakeLowercase{\textit{et al.}}: Bare Demo of IEEEtran.cls for Journals}
\newtheorem{prop}{Proposition}
\newtheorem{property}{Property}
\newtheorem{mydef}{Definition}
\newtheorem{theorem}{Theorem}
\renewcommand{\algorithmicrequire}{\textbf{Input:}}
\renewcommand{\algorithmicensure}{\textbf{Output:}}
\maketitle

\begin{abstract}
Metric learning has been successful in learning new metrics adapted to numerical datasets. However, its development on categorical data still needs further exploration. In this paper,  we propose a method, called CPML for  \emph{categorical projected metric learning}, that tries to efficiently~(i.e. less computational time and better prediction accuracy) address the problem of metric learning in categorical data. We make use of the Value Distance Metric to represent our data and propose new distances based on this representation. We then show how to efficiently learn new metrics. We also generalize several previous regularizers through the Schatten $p$-norm and provides a generalization bound for it that complements the standard generalization bound for metric learning. Experimental results show that our method provides state-of-the-art results while being faster.
\end{abstract}


%
\IEEEpeerreviewmaketitle

\section{Introduction}
Metric (or distance metric) learning represents an essential task for many machine learning problems. Relying on appropriate distance metrics can boost the performance of many learning algorithms, such as $k$-Nearest Neighbor, of which its success is largely depended on the distance metric of the points closest to a given point. Similarly, in $k$-means clustering, the shortest distance between a data point and all cluster centers also determines its cluster assignment. Several important studies have been conducted in this area, including the \emph{information theoretic metric learning} (ITML) approach \cite{davis2007information}, the \emph{large margin nearest neighbour} (LMNN) approach\cite{weinberger2006distance}, or the \emph{pseudo-metric online learning algorithm} (POLA) \cite{shalev2004online}.


However, in many cases, numerical features usually come along with categorical ones that also contain discriminative information. For instance, the categorical features of educational level and marriage status represent valuable information in the credit card fraud detection problem. The standard method to deal with these categorical features is to treat them as numerical ones by transforming them into binary vectors. However, the feature number is increased in a polynomial rate. Unsupervised learning methods for categorical distances, as  \cite{Boriah_similaritymeasures}, usually rely on the simple overlapping similarity, that varies from the simple counting, through co-occurrence frequency to entropy. When label information is available,  the supervised-learning of categorical measures \cite{Cheng20041471,kernel_density_metric_learning} is further developed. However, these methods either ignore the correlation between data samples, or come at a heavy computational cost. In addition, none of these studies provide theoretical guarantees on the generalization bound of the learned metric.




To address the above problems, in our work, we put forward a new method, namely \emph{categorical projected metric learning} (CPML), to efficiently learn metrics on categorical features and utilize them in real classification tasks. First, we employ the standard \emph{value distance metric} (VDM) \cite{stanfill1986toward} to project each feature value into a class-based vector. Then these vectors are re-arranged to define new distances relying on the correlation between features. These new defined distances are further utilized in $k$-Nearest Neighbor classification tasks. Comparing to previous methods, our approach is superior in terms of computational cost, without loss of classification accuracy. It also comes with theoretical guarantees that ensure its reliability.

To achieve this, we apply the Schatten $p$-norm ($p\ge 1$) to regularize the eigenvalues of the metric and promote low rank solutions. Several popular regularizers are special cases of this Schatten $p$-norm; $p=1$ refers to the trace norm, $p=2$ corresponds to the Frobenius norm and $p=\infty$ represents exactly the maximum eigenvalue norm \cite{ying2012distance}. Correspondingly, we provide the generalization bound for this Schatten $p$-norm ($p \ge 1$), as a supplement for the standard generalization bound in metric learning literature \cite{DBLP:journals/corr/abs-1207-5437}.

On the experimental part, we test the performance of our model in different scenarios. By adding different number of noisy features, our model is shown to be able to correctly identify the noisy features and "denoise" them. By testing the running time in different data sizes and class numbers, we show that the class number hardly  influences our model's running time. Lastly, detailed results obtained on synthetic and real world data sets confirm our models' competitive results against other benchmark models.

The remainder of this paper is organized as follows. In Section \ref{sec_2}, we introduce the \emph{value distance metric} (VDM) method and the general framework of metric learning. We then we propose the \emph{categorical projected metric learning} (CPML) framework aiming at efficiently learning metrics on categorical features for classification tasks. A generalization bound for the general Schatten $p$-norm ($p\ge1$) is provided in Section \ref{sec_4}. After a literature review in Section \ref{sec_5}, we provide and discuss experimental results in Section \ref{sec_exps}. The conclusion and future work can be found in Section \ref{sec_7}.

\section{Preliminary Knowledge} \label{sec_2}

Throughout this study, we will use the notations provided in Table \ref{t_3}.
\begin{table}[htbp]
\caption{Notation table} \label{t_3}
\centering
\begin{tabular}{c|c}
\hline
Notation & Explanation \\
\hline
$n$ & Number of data points \\ \hline
$D$ & Number of features \\ \hline
$C$ & Number of classes \\ \hline
$M$ & $D\times D$ matrix \\
& (metric to be learned) \\ \hline
$s_d$ & Number of possible values for feature $d$ \\ \hline
$s_{\text{max}}$ & Maximum number of possible values for all the features \\ \hline
\multirow{2}{*}{$\phi(\boldsymbol{x}_i)$} & VDM-based projection of example $\boldsymbol{x}_i$ \\
& ($D\times C$ matrix) \\ \hline
\multirow{2}{*}{$\phi(\boldsymbol{X})$} & VDM-based projection of the dataset $\boldsymbol{X}$ \\
& ($D\times C\times n$ tensor) \\ \hline
\multirow{2}{*}{$N_{cd}(f)$} & Number of times value $f$ is observed \\
& for feature $d$ in class $c$ \\ \hline
$R_+^{D\times D}$ & Set of semi-definite positive matrices \\ \hline
\end{tabular}
\end{table}

As is common setting, our training data includes the observations $\boldsymbol{X}=(\boldsymbol{x}_1,\ldots, \boldsymbol{x}_n)$ and the corresponding labels $\{y_i\}_{i=1}^n$. Furthermore, $\forall i$, $\boldsymbol{x}_i$ contains $D$ categorical features, i.e. $\boldsymbol{x}_i=[\boldsymbol{x}_{i1}, \ldots, \boldsymbol{x}_{iD}]^{\top}$, each $\boldsymbol{x}_{id}$ taking value in $\{1, \ldots, s_d\}$, where $s_d$ denotes the number of possible values for feature $d$.

We now introduce the representation we use for categorical features and provide the general framework for metric learning.

\subsection{Value Distance Metric}

The \emph{value distance metric} (VDM) is a method for representing categorical features into a $C$-dimensional normalized vector ($C$ corresponds to the number of classes, see Table~\ref{t_3}). $\forall d\in\{1, \ldots, D\}$, VDM partitions the whole dataset into $s_d$ subgroups, where the data points in the same subgroup have the same value on feature $d$. Then, VDM would histogram the data points according to their corresponding class labels and the histogram is normalized to represent the $d^{\hbox{th}}$ feature's class distribution.

More specifically, the feature $\boldsymbol{x}_{id}$ is first transformed into the $C$-length vector as follows:
\begin{equation} \label{VDM_eq_1}
\phi (\boldsymbol{x}_{id}=f)=[\widehat{P}(y=1|\boldsymbol{x}_{id}=f),\ldots , \widehat{P}(y=C|\boldsymbol{x}_{id}=f)]^{\top} 
\end{equation}
where $\widehat{P}(y=c|\boldsymbol{x}_{id}=f)$ is the estimate of the probability of having class $c$ when the $d^{\hbox{th}}$ feature of point $i$ has value $f$. It is defined as:
\begin{equation}
\widehat{P}(y=c|\boldsymbol{x}_{id}=f)=\frac{N_{cd}(f)}{\sum_{c=1}^{C}N_{cd}(f)} 
\end{equation}
Here $N_{cd}(f)=\sum_{i=1}^N\pmb{1}(x_{id}=f, y_i = c)$ denotes the number of times the feature value $f$ occurs in class $c$ for the $d^{\hbox{th}}$ feature, and $\sum_{c=1}^{C}N_{cd}(f)=\sum_{i=1}^N\pmb{1}(x_{id}=f)$ refers to the total appearances of feature value $f$ for the $d^{\hbox{th}}$ feature. VDM is thus a class-based projection, inspired from the original value distance metric \cite{stanfill1986toward,Cheng20041471}. Recent work in \cite{kernel_density_metric_learning} has also employed this projection.

{We take the credit risk data as an example. In Table~\ref{VDM_example}, we have a set of $6$ persons. The occupations, educations and marital status of these persons are taken as features, and the credit risk level is taken as labels. For person $1$'s occupation feature, which is Accountant, we can first estimate the probability of different credit risk for person being the accountant as follows:
\begin{align}
&\widehat{P}(y=\text{Low}|\pmb{x}_{11}=\text{Accountant}) \nonumber \\
=& \frac{\sum_{i=1}^6\pmb{1}(y_i=\text{Low}, \pmb{x}_{i1}=\text{Accountant})}{\sum_{i=1}^6\pmb{1}( \pmb{x}_{i1}=\text{Accountant})}=\frac{1}{2} \nonumber
\end{align}
Correspondingly, we can represent Eq.~(\ref{VDM_eq_1}) as:
\begin{align}
\phi(\pmb{x}_{i1}=\text{Accountant}) = (\frac{1}{2}, 0, \frac{1}{2})
\end{align}

\begin{table}[t]
\caption{Categorical data example} \label{VDM_example}
\centering
\begin{tabular}{|c|c|c|c|c|}
\hline
ID  & Occupation & Education & Marital & Risk\\
\hline
$1$  & Accountant & Bachelor & Married & Low\\
\hline
$2$  & Doctor & Master & Married & Low\\
\hline
$3$  & Plumber & TAFE & Single & High\\
\hline
$4$  & Plumber & High school & Single & Middle\\
\hline
$5$  & Doctor & Master & Married & Middle\\
\hline
$6$  & Accountant & Master & Single & High\\
\hline
\end{tabular}
\end{table}}

\subsection{Metric Learning}
\emph{Metric learning} naturally arises in the question of how to assess the similarity of different objects. Its corresponding distance function is usually set as the Mahalanobis distance, with the inverse covariance matrix as the unknown variables. With the prior knowledge of class labels or side information, we are trying to find an optimal metric that aims at minimizing the number of errors made.
\begin{equation}
\arg\min_{M} f(M)+\lambda r(M) \nonumber
\end{equation}
Here $f(\cdot)$ is the loss function, $r(\cdot)$ is the regularization function and $\lambda$ is the tuning parameter balancing the loss incurred and the model complexity. 

Our approach, as well as most metric learning approaches, fits within this general setting.

\section{Categorical Metric Learning}


VDM transforms each data point $\boldsymbol{x}_i$ into a $D\times C$ matrix, denoted as $\phi(\boldsymbol{x}_i)\in \mathbb{R}^{D\times C}, \forall i\in\{1, \ldots, n\}$. Each row in $\phi(\boldsymbol{x}_i)$ is an estimation of the class distribution, so that $\phi(\boldsymbol{x}_i)\cdot\boldsymbol{1}_{C\times 1} = \boldsymbol{1}_{D\times 1}$. The $c^{th}$ column, $\phi_c(\boldsymbol{x}_i)$, in $\phi(\boldsymbol{x}_i)$ represents the popularity of class $c$ in different features, which could be denoted as $\phi(\boldsymbol{x}_i)=[\phi_1(\boldsymbol{x}_i), \phi_2(\boldsymbol{x}_i), \ldots, \phi_C(\boldsymbol{x}_i)]$. The whole projected dataset is $
\phi(\boldsymbol{X})\in \mathbb{R}^{D\times C\times n}$.

Based on this representation, we define two distances that take into account the correlations between features inside each class. They differ in the way of treating the metric learned for different classes (different metrics for different classes are learned in one case, whereas a single metric for all classes is learned in the other case).

\begin{mydef} {\bf CPm}: The categorical projected multi (CPm) distance considers the features' individual metrics $\{M_c\}_{c=1}^C$ among different classes and is defined by:
\begin{equation} \label{eq_cpm}
\begin{split}
& d_M(\boldsymbol{x}_i, \boldsymbol{x}_j)\\
= &\sum_{c=1}^C{Tr}\left((\phi_c(\boldsymbol{x}_i)-\phi_c(%
\boldsymbol{x}_j))(\phi_c(\boldsymbol{x}_i)-\phi_c(\boldsymbol{x}_j))^{\top}M_c^{\top}\right) \\
=&\sum_{c=1}^C{Tr}\left(A^{ij}_c M_c^{\top}\right) = \sum_{c=1}^C\sum_{p, q}A^{ij}_{c,pq}M_{c,pq}
\end{split}
\end{equation}
\end{mydef}
Here $A^{ij}_c = (\phi_c(\boldsymbol{x}_i)-\phi_c(\boldsymbol{x}_j))(\phi_c(\boldsymbol{x}_i)-\phi_c(\boldsymbol{x}_j))^{\top}, M_c\in \mathbb{R}_+^{D\times D}$.

\begin{mydef} {\bf CPs}: The categorical projected single (CPs) distance considers the features' correlation by assuming that all the classes share the same metric, and is defined by:
\begin{equation} \label{eq_cps}
\begin{split}
&d_M(\boldsymbol{x}_i, \boldsymbol{x}_j) \\
= &{Tr}\left((\phi(\boldsymbol{x}_i)-\phi(%
\boldsymbol{x}_j))(\phi(\boldsymbol{x}_i)-\phi(\boldsymbol{x}_j))^{\top}M^{\top}\right) \\
=&{Tr}(A^{ij}M^{\top}) = \sum_{p, q}A^{ij}_{pq}M_{pq}
\end{split}
\end{equation}
\end{mydef}
Here $A^{ij} = (\phi(\boldsymbol{x}_i)-\phi(\boldsymbol{x}_j))(\phi(\boldsymbol{x}_i)-\phi(\boldsymbol{x}_j))^{\top}, M\in \mathbb{R}_+^{D\times D}$.


We assume the positive semi-definite property of $\{M_c\}_{c=1}^C$ and $M$ to ensure the the non-negativeness of $d_M(\boldsymbol{x}_i, \boldsymbol{x}_j)$. The following property (see the proof in the appendix) furthermore shows that the above definitions correspond to valid distances.

\begin{property}
Assume $M\succeq 0$, then $d_M(\boldsymbol{x}_i, \boldsymbol{x}_j)\ge 0$, and, $\forall \boldsymbol{x}_i, \boldsymbol{x}_j, \boldsymbol{x}_k$, we have $d_M(\boldsymbol{x}_i, \boldsymbol{x}_j)\le d_M(\boldsymbol{x}_i, \boldsymbol{x}_k) + d_M(\boldsymbol{x}_k, \boldsymbol{x}_j) $.
\end{property}

It is obvious $d_M(\boldsymbol{x}_i, \boldsymbol{x}_i) = 0$ and $d_M(\boldsymbol{x}_i, \boldsymbol{x}_j)=d_M(\boldsymbol{x}_j, \boldsymbol{x}_i)$. Thus, our definitions of CPm and CPs $d_M(\boldsymbol{x}_i, \boldsymbol{x}_j)$ in Eq. (\ref{eq_cpm}) and (\ref{eq_cps}) correspond to valid metrics.

\subsection{Objective Function \& Optimization}

From the class information in the data set, one can derive a constraint set based on triplets of points, $\mathcal{T}=\{(i,j, k)|d_M(\boldsymbol{x}_i, \boldsymbol{x}_j)<d_M(\boldsymbol{x}_i,
\boldsymbol{x}_k)\}$, that indicates that any point should be closer to points of the same class than to points of other classes. From this constraint set, our task is to find an optimal $M$ such that the \emph{empirical loss} $\boldsymbol{\varepsilon}_{\mathcal{T}}(M)$ minimized. The empirical loss $\boldsymbol{\varepsilon}_{\mathcal{T}}(M)$ is here defined as:
\begin{equation}  \label{eq_2}
\boldsymbol{\varepsilon}_{\mathcal{T}}(M)=\frac{1}{|\mathcal{T}|}%
\sum_{(i,j,k)\in\mathcal{T}} \left[d_M(\boldsymbol{x}_i, \boldsymbol{x}%
_j)+b-d_M(\boldsymbol{x}_i, \boldsymbol{x}_k)\right]_+
\nonumber
\end{equation}
$[\cdot]_+$ is the hinge loss function, and $b$ is the margin parameter, often set to $b = 1$.

The above setting can be used for constraint sets based on pairs of points, $\mathcal{P}=\{(i,j)|d_M(\boldsymbol{x}_i, \boldsymbol{x}_j)<b\}$, in which case the empirical loss takes the form:
\begin{equation}  \label{eq_pc}
\boldsymbol{\varepsilon}_{\mathcal{P}}(M)=\frac{1}{|\mathcal{P}|}%
\sum_{(i,j)\in\mathcal{P}} \left[1+d_M(\boldsymbol{x}_i, \boldsymbol{x}%
_j)-b\right]_+
\nonumber
\end{equation}
It can also be used for quadratic constraint sets, based on 4-uples of data points, $\mathcal{Q}=\{(i,j,k,l)|d_M(\boldsymbol{x}_i, \boldsymbol{x}_j)<d_M(\boldsymbol{x}_k,\boldsymbol{x}_l)\}$, in which case the empirical loss takes the form:
\begin{equation}  \label{eq_qc}
\boldsymbol{\varepsilon}_{\mathcal{Q}}(M)=\frac{1}{|\mathcal{Q}|}\sum_{(i,j,k, l)\in\mathcal{Q}} \left[d_M(\boldsymbol{x}_i, \boldsymbol{x}_j)+b-d_M(\boldsymbol{x}_k, \boldsymbol{x}_l)\right]_+
\nonumber
\end{equation}
{In real world applications, no single constraint would dominate the other two in different scenarios. When two points from different classes can indeed be very close or far away to each other, the triplet constraint set $\mathcal{T}$ and the pair constraint set $\mathcal{P}$ do not make strong assumptions in this case, whereas the quadratic constraint $\mathcal{Q}$ might set these assumptions. We believe that both $\mathcal{P}$ and $\mathcal{T}$ define valuable information on which to learn a new metric. We present here the solution of the optimization problem based on $\mathcal{T}$, but the same development can be used for $\mathcal{P}$.}

\begin{algorithm}[hbp]
\caption{The CPML framework }
\label{alg_1}
\begin{algorithmic}
    \REQUIRE $\phi(\boldsymbol{X})$: $D\times C\times n$ projected tensor for the whole dataset; \\ $\alpha_t$, stepsize at the $t^{\textrm{th}}$ iteration; $\alpha$, stepsize's decreasing rate\\
    $\lambda$: regularization parameter; $\mathcal{T}$: constraint set for the whole dataset
    \ENSURE the resulted metrics $M$
    \STATE Initialize $M= I_{D\times D}$
    \STATE Compute $A$ valuees between each data pair
    \REPEAT
    \STATE Compute the violated constraint set $\widehat{\mathcal{T}}$
    \STATE Compute the subgradient for $f$:
    {\small \begin{equation}
    \frac{\partial}{\partial M} f(M)=\frac{\partial r}{M}(M) + \frac{\lambda}{|\mathcal{T}|}\sum_{(i,j,k)\in\widehat{\mathcal{T}}}\left[A^{ij}-A^{ik}\right]
    \end{equation}}
    \STATE Use {\bf backtrack line search} method to determine $\alpha_t$';
    \STATE Update $M$ as: $M_{t+1} = M_t-\alpha_t\frac{\partial}{\partial M} f(M)$;
    \STATE Project $M$ back into the Positive Semi-Definite cone
    \UNTIL{converge}
\end{algorithmic}
\end{algorithm}

By choosing a suitable metric regularizer $r(M)$, our problem amounts to minimize the objective function $f(M)$:
\begin{equation}
\arg \min_{M}f(M)=\arg \min_{M}\left\{ \boldsymbol{\varepsilon }_{\boldsymbol{x}}(M)+\lambda \cdot r(M)\right\} \nonumber
\end{equation}
where $\lambda$ is the regularization parameter.

{The choice of the metric regularizer affects the structure of the solution learned. For instance, the $L_1$-norm promotes sparse metric, while the trace norm encourages  metrics with low rank; the Frobenius norm ($\Vert M\Vert _{2}^{2}=\sum_{pq}M_{pq}^{2}$), on the other hand, tends to yield robust solutions.}

If $r(M)$ is a convex function, the objective function $f(M)$ is also convex with respect to $M$. As the empirical hinge loss is non-differentiable at $0$, we apply the Projected Subgradient Descent method to seek the optimal value of $M$. The subgradient of $f(M)$ is composed of two terms: one in the regularization, i.e. the matrix $\frac{\partial r}{M}(M)$, and the other in the empirical loss function. The subgradient of the empirical loss is the sum of the sub gradient of the hinge loss in each constraint in $\mathcal{T}$. For each $(i,j,k)\in\mathcal{T}$, the subgradient direction is $0$ if the loss is $0$, i.e. $d_M(\boldsymbol{x}_i, \boldsymbol{x}_j)+ b \le d_M(\boldsymbol{x}_i, \boldsymbol{x}_k)$. For $d_M(\boldsymbol{x}_i, \boldsymbol{x}_j)+ b > d_M(\boldsymbol{x}_i, \boldsymbol{x}_k)$, we get:
\begin{equation}
\begin{split}
&\frac{\partial \left[d_M(\boldsymbol{x}_i, \boldsymbol{x}_j)+ b -d_M(%
\boldsymbol{x}_i, \boldsymbol{x}_k)\right]_+}{\partial M} \\
= & \frac{\partial\left[Tr(A^{ij}M)+ b -Tr(A^{ik}M)\right]_+}{\partial M} \\
= & A^{ij}-A^{ik} \\
\nonumber
\end{split}%
\end{equation}

%

\begin{algorithm}[hbp]
\caption{Backtrack line search for determining the stepsize \cite{Nocedal2006NO} }
\label{alg_2}
\begin{algorithmic}
    \REQUIRE $\phi(\boldsymbol{X}), M$: projected matrix, current metric\\
    $f$: objective function\\
    $\frac{\partial}{\partial M} f(M)$: gradient of the objective function $f$\\
    $\alpha$: stepsize decreasing rate, usually set to $0.1$
    \ENSURE the stepsize $\alpha_t$
    \STATE Initialize $\alpha_t$ = 1;
    \REPEAT
    \STATE $\alpha_t = \alpha_t\cdot \alpha$;
    \UNTIL{$f(M-\alpha_t\frac{\partial}{\partial M} f(M))\le f(M)-\frac{\alpha_t}{2}\|\frac{\partial}{\partial M} f(M)\|^2$}
\end{algorithmic}
\end{algorithm}

Thus, the subgradient of $f(M)$ is:
\begin{equation}
\frac{\partial}{\partial M}f(M) = \frac{\partial r}{M}(M) + \frac{\lambda}{|\mathcal{T}|} \sum_{(i,j,k)\in\widehat{\mathcal{T}}}\left[A^{ij}-A^{ik}\right]
\nonumber
\end{equation}
where $\widehat{\mathcal{T}} = \{ {(i,j,k)|d_M(\boldsymbol{x}_i, \boldsymbol{x}_j)+ b > d_M(\boldsymbol{x}_i, \boldsymbol{x}_k)}\}$ denotes the set of triplets for which the constraint is violated.

After each gradient step $M_{t+1}=M_t-\alpha_t\frac{\partial}{\partial M_t} f(M_t)$, we need to project $M$ back to the positive semi-definite cone. This is conducted by setting the negative eigenvalues in $M_t$ to be $0$.

The complete process is described in Algorithm \ref{alg_1}. It is important to note that the $A^{ij}$ values can be computed before the iteration, which reduces the computational cost.

%

%

\section{Rademacher Complexity and Schatten $p$-Norms} \label{sec_4}

In choosing the regularizer $r(M)$, we here rely on a Schatten $p$-norm of the learned metric $[Tr(M^p)]^{\frac{1}{p}}, p\ge 1$, that generalizes several well-known metric-regularizers.
The study in \cite{DBLP:journals/corr/abs-1207-5437} gives a generalization bound on the pair-comparison empirical loss, defined by:
%
\begin{equation}
\boldsymbol{\varepsilon}_{\mathcal{T}}(M, b)= \frac{1}{n(n-1)}\sum_{i\neq j}\left[1+r_{ij}(d_M(x_i,x_j)-b)\right]_+ \nonumber
\end{equation}
Here $r_{ij}=1$ if $y_i=y_j$, otherwise $r_{ij}=1$; $b$ is the margin parameter. One expects that $r_{ij}=1$ if $d_M(x_i,x_j)\le b$, and $r_{ij}=0$ otherwise.

Thus, the objective function to be minimized is:
\begin{equation}
\frac{1}{n(n-1)}\sum_{i\neq j}\left[1+r_{ij}(d_M(x_i,x_j)-b)\right]_++\lambda\|M\|^2 \nonumber
\end{equation}
where $\|M\|^2$ is the metric regularizer, and $\lambda>0$ the regularization parameter.

From this, the study in \cite{DBLP:journals/corr/abs-1207-5437} derives the following generalization bound, that holds with probability at least $(1-\delta)$:
\begin{equation}
\begin{split}
 \boldsymbol{\varepsilon}(M, b)-\boldsymbol{\varepsilon}_{\mathcal{T}}(M, b)
 &\le\frac{4R_n(M)}{\sqrt{\lambda}}+\frac{4(3+2X_*/\sqrt{\lambda})}{\sqrt{n}}\\
 &+2(1+X_*/\sqrt{\lambda})
 \left(\frac{2\ln\left(\frac{1}{\delta}\right)}{n}\right)^{\frac{1}{2}}
 \nonumber
\end{split}
\end{equation}
where $\widehat{R}_n(M)$ denotes the empirical Rademacher complexity, defined as:
\begin{equation}
\widehat{R}_n(M) = \frac{1}{\lfloor\frac{n}{2}\rfloor}\mathbb{E}_{\sigma}\|\sum_{i=1}^{\lfloor\frac{n}{2}\rfloor}
\sigma_iX_{i({\lfloor\frac{n}{2}\rfloor}+i)}\|_*
\nonumber
\end{equation}
and $X_*=\sup_{x, x'\in\mathcal{X}}\|(x-x')(x-x')^{\top}\|_*$ ($X_*$ measures the diameter of the domain of $\mathcal{X}$). $\|\cdot\|_*$ denotes the dual norm for a given norm $\|\cdot\|$. 

%

However, the study in \cite{DBLP:journals/corr/abs-1207-5437} does not consider the Schatten $p$-norm as a metric regularizer, and we provide here a bound on the Rademacher complexity of the Schatten $p$-norm regularizer for the case $p\ge1$. This bound complements the study in \cite{DBLP:journals/corr/abs-1207-5437} for the pair-comparison case. To do this, we first define the expectation of the empirical Rademacher complexity:
\begin{equation}
\begin{split}
& R(M) = \mathbb{E}_{\boldsymbol{z}}\widehat{R}_n(M) = \frac{1}{\lfloor\frac{n}{2}\rfloor}\mathbb{E}_{\boldsymbol{z}, \sigma}\|\sum_{i=1}^{\lfloor\frac{n}{2}\rfloor}
\sigma_iX_{i({\lfloor\frac{n}{2}\rfloor}+i)}\|_*
\end{split}
\nonumber
\end{equation}
Then, we have the following theorem.
\begin{theorem}
The Rademacher Complexity of our distances in the Schatten $p$-norm ($p\ge 1$) in the pair-comparison empirical loss case is bounded by:
\begin{equation}
R(M)\le\left\{\begin{array}{ll}
D^{\frac{1}{2}-\frac{1}{p}}\cdot\frac{2X_*}{\sqrt{n}}, & 2\le p; \\
D^{1-\frac{1}{p}}\cdot\frac{2X_*}{\sqrt{n}}, & 1\le p<2. \\
\end{array}\right.
\end{equation}
\end{theorem}

\begin{proof}
The dual norm of the Schatten $p$-norm is the Schatten $q$-norm, where $q$ satisfies $\frac{1}{p}+\frac{1}{q}=1$. Let us first assume that $p\ge 2$; then $1\le q\le 2$. Assume $\lambda_{1, \sigma, i}, \ldots, \lambda_{d, \sigma, i}$ are the eigenvalues of the matrix $\sum_{i=1}^{\lfloor\frac{n}{2}\rfloor} \sigma_iX_{i({\lfloor\frac{n}{2}\rfloor}+i)}$, we have
\begin{equation} \label{gb_1}
\begin{split}
R(M)= & \frac{1}{\lfloor\frac{n}{2}\rfloor}\mathbb{E}_{\boldsymbol{z}, \sigma}\|\sum_{i=1}^{\lfloor\frac{n}{2}\rfloor}
\sigma_iX_{i({\lfloor\frac{n}{2}\rfloor}+i)}\|_q\\
= & \frac{1}{\lfloor\frac{n}{2}\rfloor}\mathbb{E}_{\boldsymbol{z}, \sigma}[\sum_{k=1}^D(\lambda_{k,\sigma, i})^{q}]^{\frac{1}{q}}\\
\overset{H\ddot{o}lder}{\le} & \frac{1}{\lfloor\frac{n}{2}\rfloor}\mathbb{E}_{\boldsymbol{z}, \sigma}
\left\{\left[\sum_{k=1}^D((\lambda_{k,\sigma, i})^q)^{\frac{2}{q}}\right]^{\frac{q}{2}}\cdot D^{\frac{2-q}{2}}\right\}^{\frac{1}{q}}\\
= & \frac{1}{\lfloor\frac{n}{2}\rfloor}\mathbb{E}_{\boldsymbol{z}, \sigma}
\left[\sum_{k=1}^D(\lambda_{k,\sigma, i})^2\right]^{\frac{1}{2}}\cdot D^{\frac{2-q}{2q}}\\
\overset{\textrm{\cite{DBLP:journals/corr/abs-1207-5437}}}{\le} & D^{\frac{2-q}{2q}}\cdot\frac{2X_*}{\sqrt{n}}
\end{split}
\end{equation}

Let us now assume that $1\le p<2$; then $q>2$. Let $\lambda_{\textrm{max}, \sigma, i} = \max\{\lambda_{1, \sigma, i}, \ldots, \lambda_{d, \sigma, i}\}$.
\begin{equation} \label{gb_2}
\begin{split}
R(M)= & \frac{1}{\lfloor\frac{n}{2}\rfloor}\mathbb{E}_{\boldsymbol{z}, \sigma}[\sum_{k=1}^d(\lambda_{k,\sigma, i})^{q}]^{\frac{1}{q}} \le\frac{1}{\lfloor\frac{n}{2}\rfloor}\mathbb{E}_{\boldsymbol{z}, \sigma}[d^{\frac{1}{q}}\cdot
\lambda_{\textrm{max}, \sigma, i}]\\
= & d^{\frac{1}{q}}\cdot\frac{1}{\lfloor\frac{n}{2}\rfloor}\mathbb{E}_{\boldsymbol{z}, \sigma}[
\lambda_{\textrm{max}, \sigma, i}] \\
\le &d^{\frac{1}{q}}\cdot\frac{1}{\lfloor\frac{n}{2}\rfloor}\mathbb{E}_{\boldsymbol{z}, \sigma}
\left[\sum_{k=1}^d(\lambda_{k,\sigma, i})^2\right]^{\frac{1}{2}}\\
\overset{\textrm{\cite{DBLP:journals/corr/abs-1207-5437}}}{\le} & d^{\frac{1}{q}}\cdot\frac{2X_*}{\sqrt{n}}
\end{split}
\end{equation}
It has to be noted here that, as we are using the union bound in this case, the result in the $1\le p<2$ case may be loose.

The conclusion is finally obtained by summarizing Equations (\ref{gb_1}) and (\ref{gb_2}). \hspace{0cm} $\Box$
\end{proof}

Armed with this result, we can now state a generalization bound for with Schatten $p$-norm in the pair-comparison case.
\begin{theorem}
$\forall 0<\delta<1$, with probability $(1-\delta)$, we have that
\begin{equation}
\begin{split}
 \boldsymbol{\varepsilon}(M, b)-\boldsymbol{\varepsilon}_{\mathcal{T}}(M, b)
 &\le\frac{8D^{\alpha(p)-\frac{1}{p}}X_*}{\sqrt{n\lambda}}+\frac{4(3+2X_*/\sqrt{\lambda})}{\sqrt{n}}\\
 &+2(1+X_*/\sqrt{\lambda})
 \left(\frac{2\ln\left(\frac{1}{\delta}\right)}{n}\right)^{\frac{1}{2}}
 \nonumber
\end{split}
\end{equation}
Here $\alpha(p)=\left\{\begin{array}{ll}
\frac{1}{2}, & 2\le p; \\
1, & 1\le p<2. \\
\end{array}\right.$
\end{theorem}

\begin{figure*}[htbp]
\begin{minipage}[b]{.24\textwidth}
\centering
{\small
\begin{displaymath}
\left( \begin{array}{cccc}
M_{11} & M_{12} & \ldots & M_{1C}\\
M_{12} & M_{22} & \ldots & M_{2C}\\
\vdots & \vdots & \ddots & \vdots\\
M_{1C} & M_{2C} & \ldots & M_{CC}\\
\end{array} \right)
\end{displaymath}}
\end{minipage}
\begin{minipage}[b]{.24\textwidth}
\centering
{\small
\begin{displaymath}
\left( \begin{array}{cccc}
M_{11} & 0 & \ldots & 0\\
0 & M_{22} & \ldots & 0\\
\vdots & \vdots & \ddots & \vdots\\
0 & 0 & \ldots & M_{CC}\\
\end{array} \right)
\end{displaymath}}
\end{minipage}
\begin{minipage}[b]{.24\textwidth}
\centering
{\small
\begin{displaymath}
\left( \begin{array}{cccc}
M_{11} & 0 & \ldots & 0\\
0 & M_{11} & \ldots & 0\\
\vdots & \vdots & \ddots & \vdots\\
0 & 0 & \ldots & M_{11}\\
\end{array} \right)
\end{displaymath}}
\end{minipage}
\begin{minipage}[b]{.24\textwidth}
\centering
{\small
\begin{displaymath}
\left( \begin{array}{cccc}
I_{D\times D} & 0 & \ldots & 0\\
0 & I_{D\times D} & \ldots & 0\\
\vdots & \vdots & \ddots & \vdots\\
0 & 0 & \ldots & I_{D\times D}\\
\end{array} \right)
\end{displaymath}}
\end{minipage}
\caption{Features' correlation structure (from left to right: KDML, CPm, CPs, Euclidean).} \label{fea_corr}
\end{figure*}

\section{Related Work} \label{sec_5}
\subsection{Metric Learning} From the seminal work of \cite{xing2002distance}, the majority of studies in metric learning focuses on numerical data. An optimal metric is usually learned, from some labeled information, in the family of Mahalanobis distances. Several metric learning methods have lead to classifiers significantly better than the one based on standard metrics as the Euclidean distance (i.e. without learning). Among such methods we can cite the \emph{large margin nearest neighbour} (LMNN) approach\cite{weinberger2006distance} that first determines the nearest neighbors of each point in the Euclidean space, then tries to move the point closer to its neighbors of the same classe while pushing it away from neighbors of other classes. The \emph{information theoretical metric learning} (ITML) \cite{davis2007information} tries to minimize the relative entropy between two multivariate Gaussian distributions under distance constraints. The \emph{maximum-eigenvalue metric learning} \cite{ying2012distance,7529190} uses the popular Frank-Wolfe algorithm to formulate the problem as a constrained maximum-eigenvalue problem, which avoids the computation of the full eigen-decomposition in each iteration.

Metric learning on numerical data usually involves a linear transformation of the original Euclidean space. Some studies, however, rely on non-linear transformations, as the $\chi^2$-LMNN and GB-LMNN  approaches \cite{kedem2012non} that respectively make use of the $\chi^2$ distance and regression trees. Other non-linear transformations involve the use of kernels \cite{globerson2007visualizing,torresani2006large}. \emph{Hamming distance metric learning} \cite{norouzi2012hamming} has recently been proposed to learn a mapping from real-valued inputs into binary ones, with which the hash function can fully utilized to enable large scalability.

On learning the distance between categorical data, \cite{Boriah_similaritymeasures} have conducted an extensive comparison between various unsupervised measures. The closest work to ours is \cite{Cheng20041471}, which is to measure correlation structure among each feature, and \cite{kernel_density_metric_learning}, which considers the full class-features' correlation structure. However, the first neglects the potential correlations between the features and both of their scalability are questionable, as they individually optimize $d\cdot n_d^2$ and $(D\cdot C)\times (D\cdot C)$ matrixes.

{We notice there are two algorithms proposed recently to learn metric on categorical data. The heterogeneous metric learning with hierarchical couplings~(HELIC)~\cite{HELIC} mainly focuses on the attribute-to-value and attribute coupling framework. Their insufficient utilization on the labels leads to degraded performance. Also, the DM3 method~\cite{DM3} considers only the frequency information within the attributes. The lack of class label incorporation makes the performance unappealing again. We show this in the experimental part.}
%

On the theoretical generalization guarantees, \cite{jin2009regularized} uses the uniform stability concept to firstly bound the deviation from \emph{true risk} to \emph{empirical risk}, under the Frobenious norm case. \cite{bian2011learning} gives the generalization bound without regularizers, with strong assumptions on the points' distribution. \cite{bellet2012robustness} has shown that robustness \cite{xu2012robustness} is necessary and sufficient to be generalized well. \cite{DBLP:journals/corr/abs-1207-5437} uses the notion of Rademacher complexity to derive the generalization bound for several regularizers. Our generalization bound on the Schatten $p$-norm ($p\ge1$) is a supplement to this work.
%
%

{\subsection{Distance between categorical data}
In addition to the VDM method we have used, there are other existing methods to quantify the distance between categorical data in the literature. Among all those methods, Hamming distance~\cite{hamming1950error} is widely known for its intuitive understanding and simplicity to implementation. The idea is to treat the same value in the categorical data as $1$, and $0$ otherwise. However, the Hamming distance lacks the ability to model dependence within the features and also the potential connection to the class information. 

The association-based metric proposed by~\cite{le2005association} uses an indirect probabilistic method. Particularly, the metric is estimated by the sum of distances between conditional probability density function~(cpdf) of other attributes given these two values, i.e., $D(x_{id}, x_{jd})=\sum_{d'}\psi(\text{cpdf}(X|x_{id})+\text{cpdf}(X|x_{jd}))$, where $\psi(\cdot, \cdot)$ is the distance function between two probability density functions and can be used as the Kullback-Leibler divergence. However, the distance is $0$ if all the attributes are independent of each other. 

The context-based metric~\cite{ienco2009context,ienco2012context} is determined by a measure of symmetrical uncertainty~(SU) between pairwise attributes. Particularly, SU is calculated as $\text{SU}(X_{d_1}, X_{d_2})=2\frac{\text{IG}(X_{d_1}|X_{d_2})}{H(X_{d_1})+H(X_{d_2})}$, where $H(X_{d_1})$ is the entropy of attribute $d_1$ and $\text{IG}(X_{d_1}|X_{d_2})$ is the information gain. The metric between two attribute value is determined by further usage of SU. Similar as the association-based metric, the context-based metric can not work well when these attributes are independent. 

Fig. \ref{fea_corr} displays four different feature correlation structures. As we can see, the CP-m and CP-s distances are in between the full correlation considered in KDML and the simple one used in the Euclidean distance.}

{\subsection{Computational complexity}
The compuational complexity of our captegorical project method is determined by two parts: calculating the VDM projection and the metric learning part. In calculating the VDM projection, we calculate the corresponding projection of all the $D$ categorical features of $n$ nodes. For each feature value, we use the corresponding class information. Thus, the total computational cost is $\mathcal{O}(nDCs_{\text{max}})$, where $s_{\text{max}}$ is the maximum number of values in one feature. 

In calculating the computational complexity of metric learning, we taken the matrix of $A$ to be fixed as it is calculated in advance. Thus, for a given iteration number $L$ and length of contraint set $T$, the computational time scales to $\mathcal{O}(LTCD^3)$ for CPML-m~(i.e. metric learning with categorical projected multi-distance) method (and $\mathcal{O}(LTD^3)$ for CPML-s~(i.e. metric learning with categorical projected single-distance) method). Here the term of $\mathcal{O}(D^3)$ refers to the spectral decomposition of $M$ such that we can manually make it positive semi-definite. In most of the other approaches (e.g. KDML), the computation complexity would scales to $\mathcal{O}(C^3D^3)$, which is almost infeasible when the number of class is large.

As a result, the computational complexity for the CPML-m and CPML-s method can be summarized as $\mathcal{O}(nDCs_{\text{max}}+LTCD^3)$ and $\mathcal{O}(nDCs_{\text{max}}+LTD^3)$.
}

\section{Experiments} \label{sec_exps}

The performance of our CPML framework is validated by experiments on synthetic dataset as well as real-world datasets. On the synthetic dataset, we mainly test the properties of CPML, e.g., the influence of the number of features, the presence of noisy features, and the running time. The real-world datasets are mainly used for evaluating the performance of different approaches. Particularly, we individually implement three baseline methods, i.e. LMNN, KDML and DM3, to the best of our understanding. For HELIC, we use the authors' kindly provided implementation. 

Further, we use {\it triplet comparison accuracy} and {\it classification accuracy} to assess the performance of the methods. The triplet constraints are built by considering that any pair ($\boldsymbol{x}_i$,$\boldsymbol{x}_j$) from the same class should have a lower distance than any pair ($\boldsymbol{x}_i$,$\boldsymbol{x}_k$) from different classes. For the \emph{triplet comparison accuracy}, we also construct these triplets on the test data and evaluate the proportion correctly predicted. For the \emph{classification accuracy}, we use Nearest Neighbor classification as the default classifier. Further, we randomly divide the data into $10$ parts and set the ratio of training:validation:testing as $6:2:2$. The trade-off parameter $\lambda$ is set as ranging from $10^{-4}$ to $10^4$. For each scenario, the experiments are run 50 times and averaged; the summary statistics (mean, standard deviation) are reported.
\begin{figure*}[htbp]
\centering
\includegraphics[width =  0.4\textwidth]{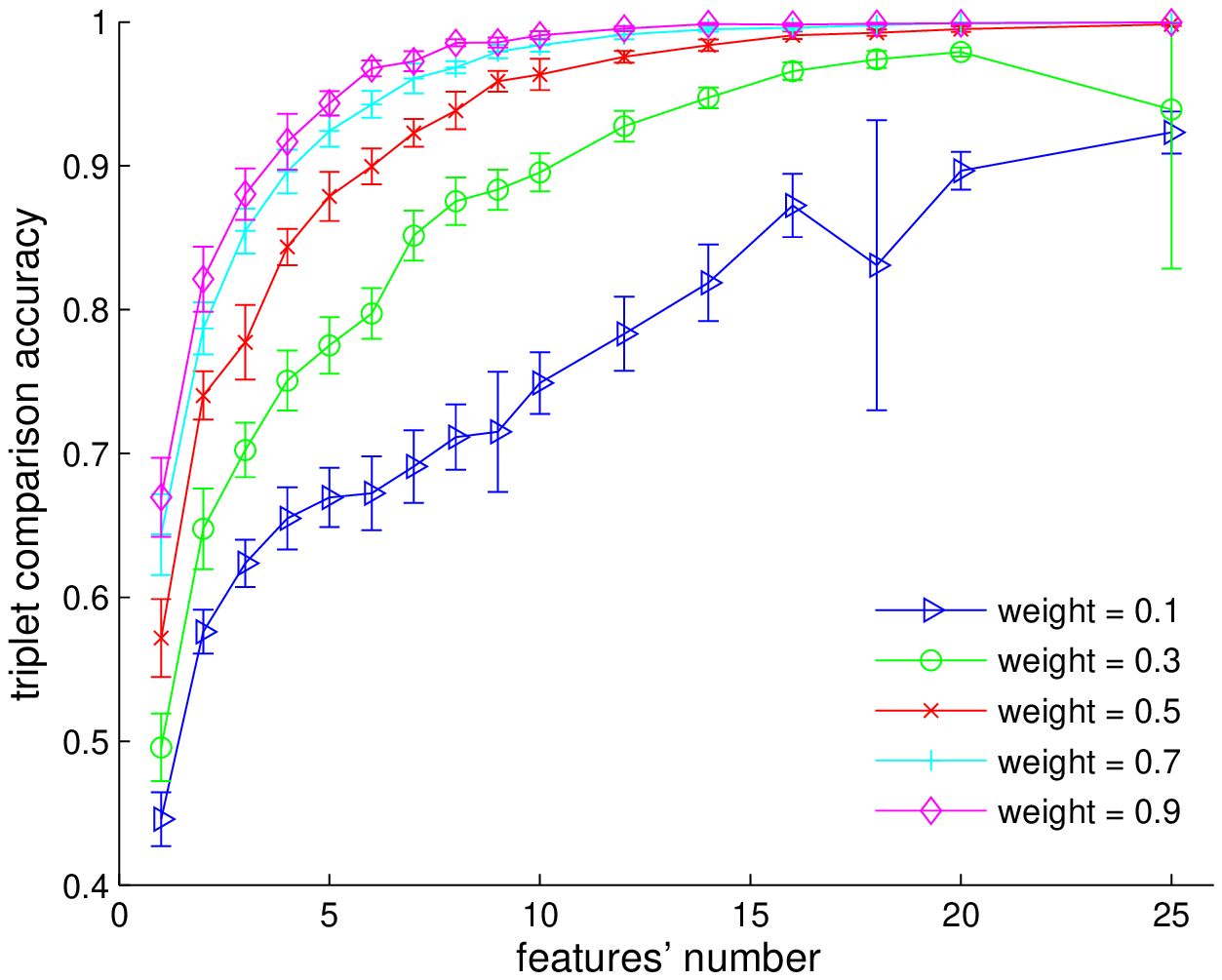}
\includegraphics[width =  0.4\textwidth]{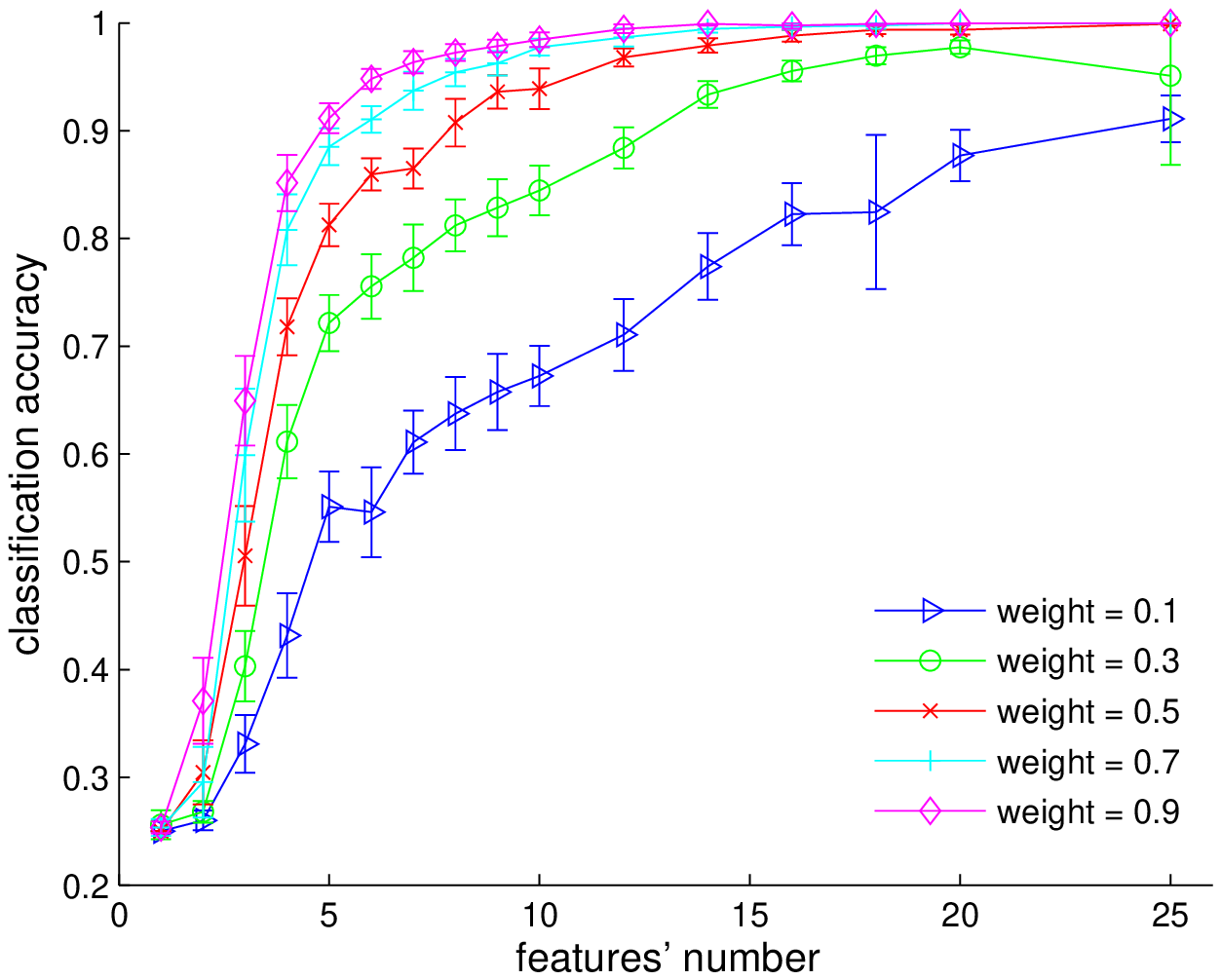}
\includegraphics[width =  0.4\textwidth]{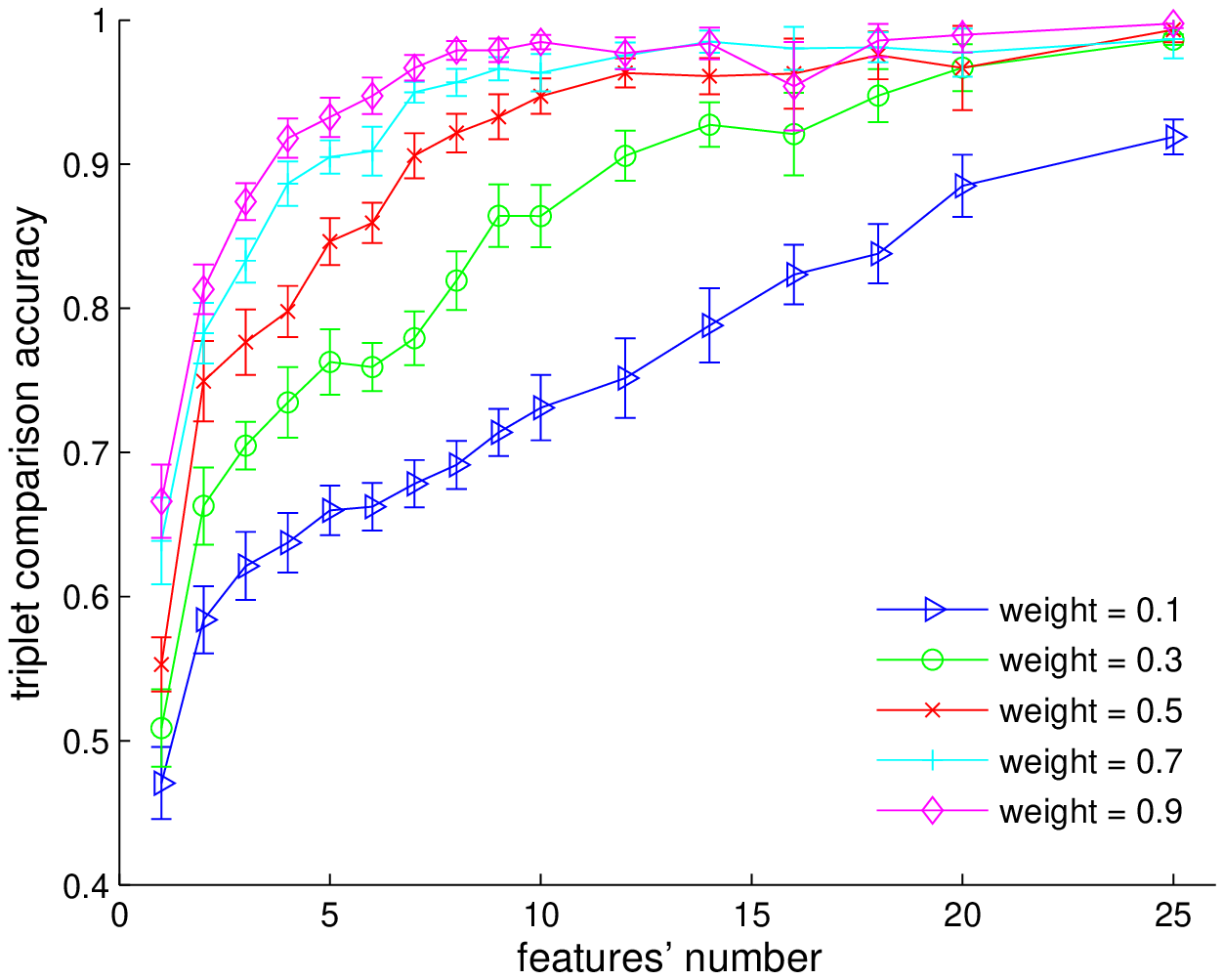}
\includegraphics[width =  0.4\textwidth]{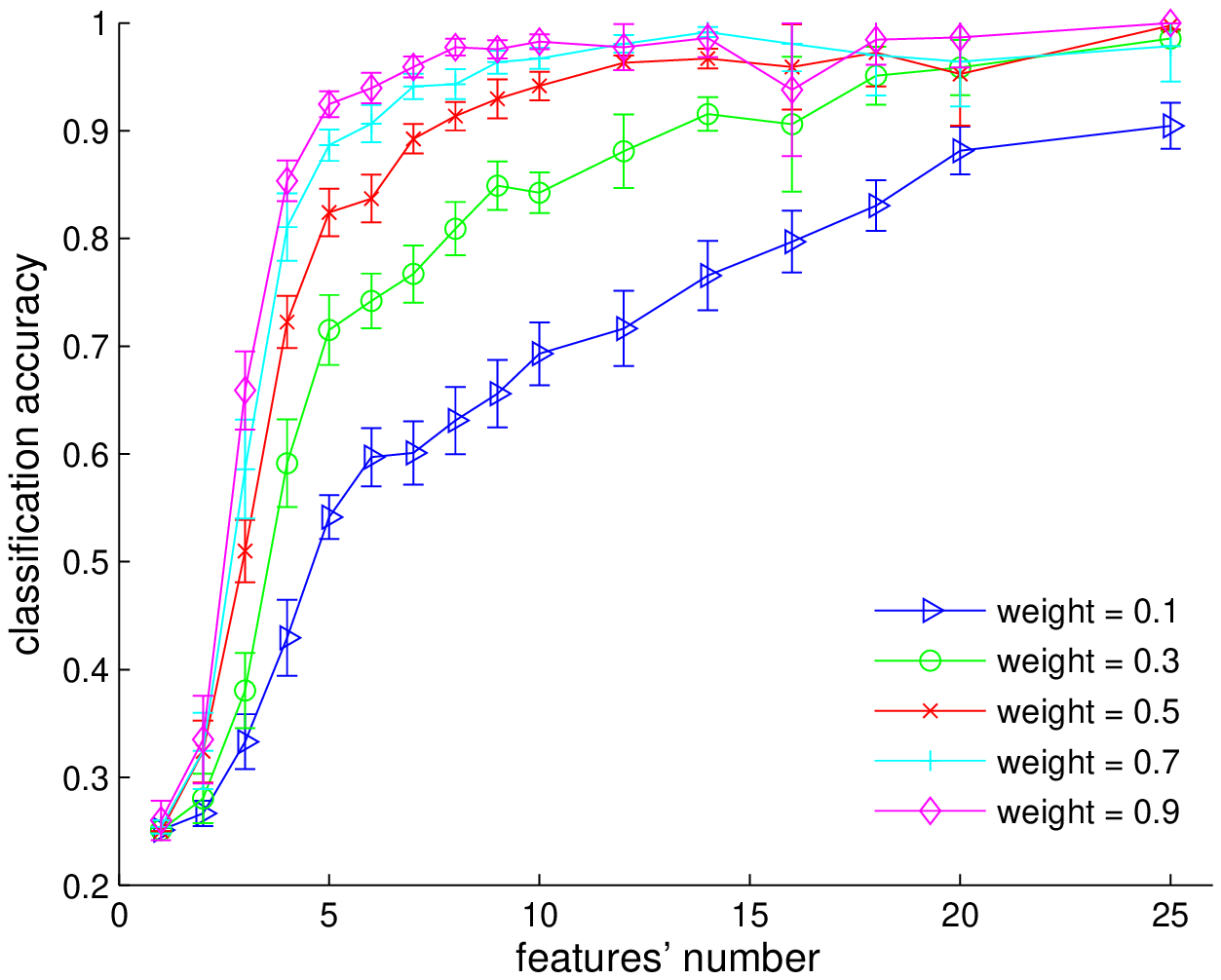}
\caption{Synthetic data classification performance (with $4$ classes). (Top: CPML-s; bottom: CPML-m.)}
\label{syn_result}
\end{figure*}
\begin{figure*}[htbp]
\centering
\includegraphics[width =  0.4\textwidth]{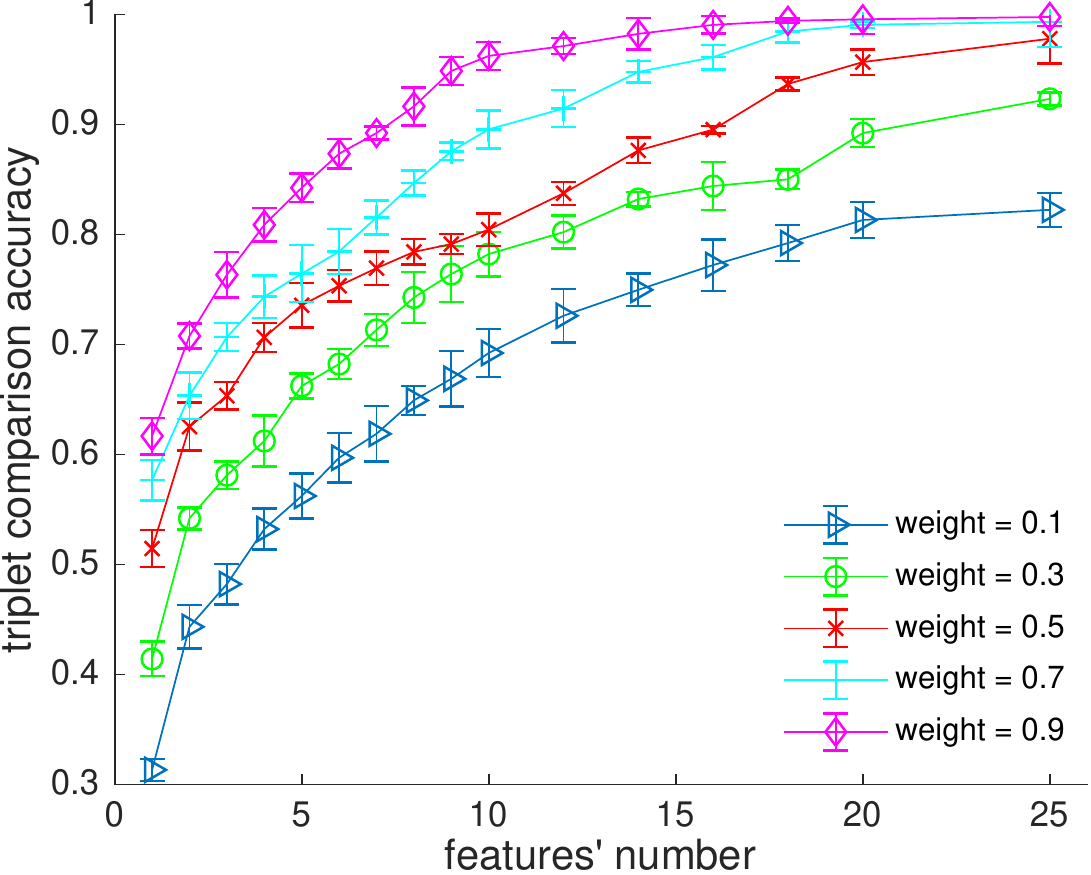}
\includegraphics[width =  0.4\textwidth]{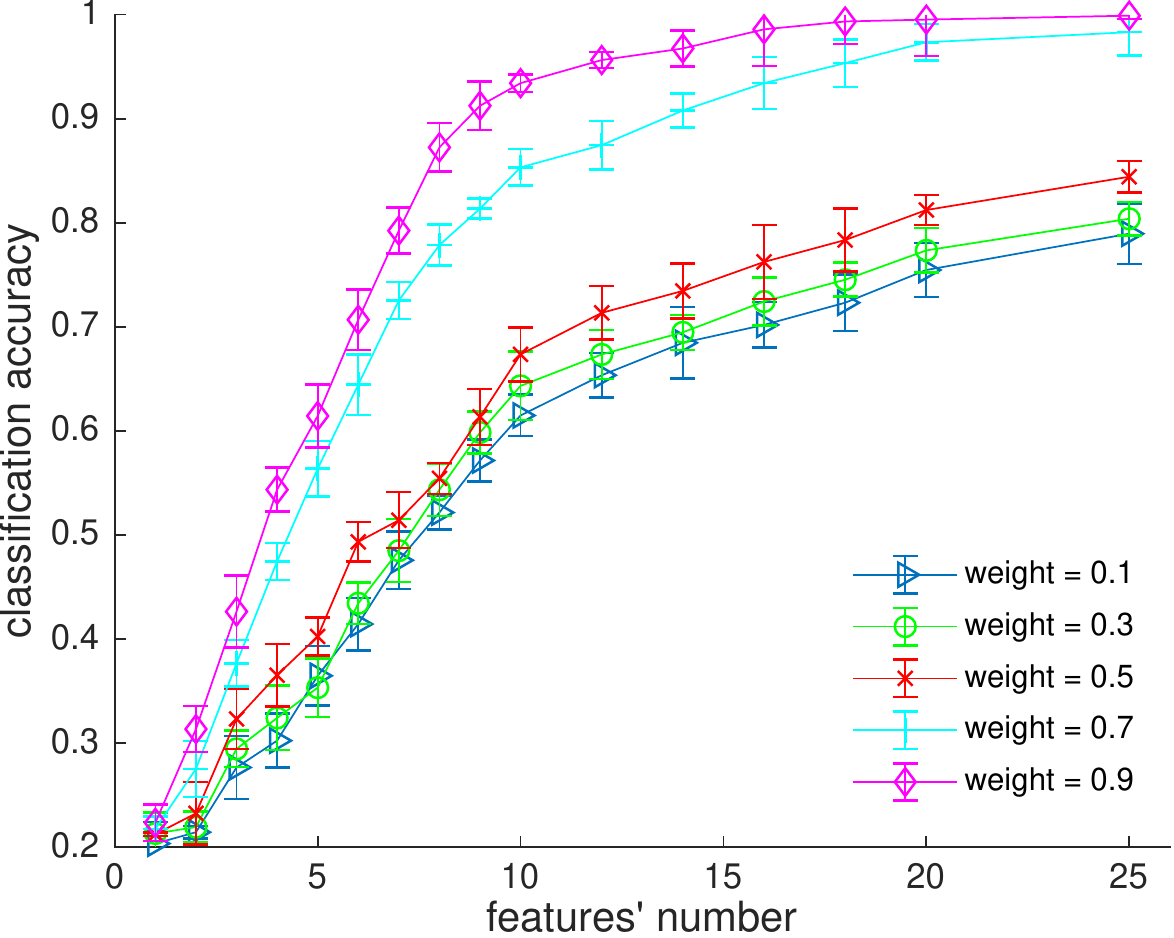}
\includegraphics[width =  0.4\textwidth]{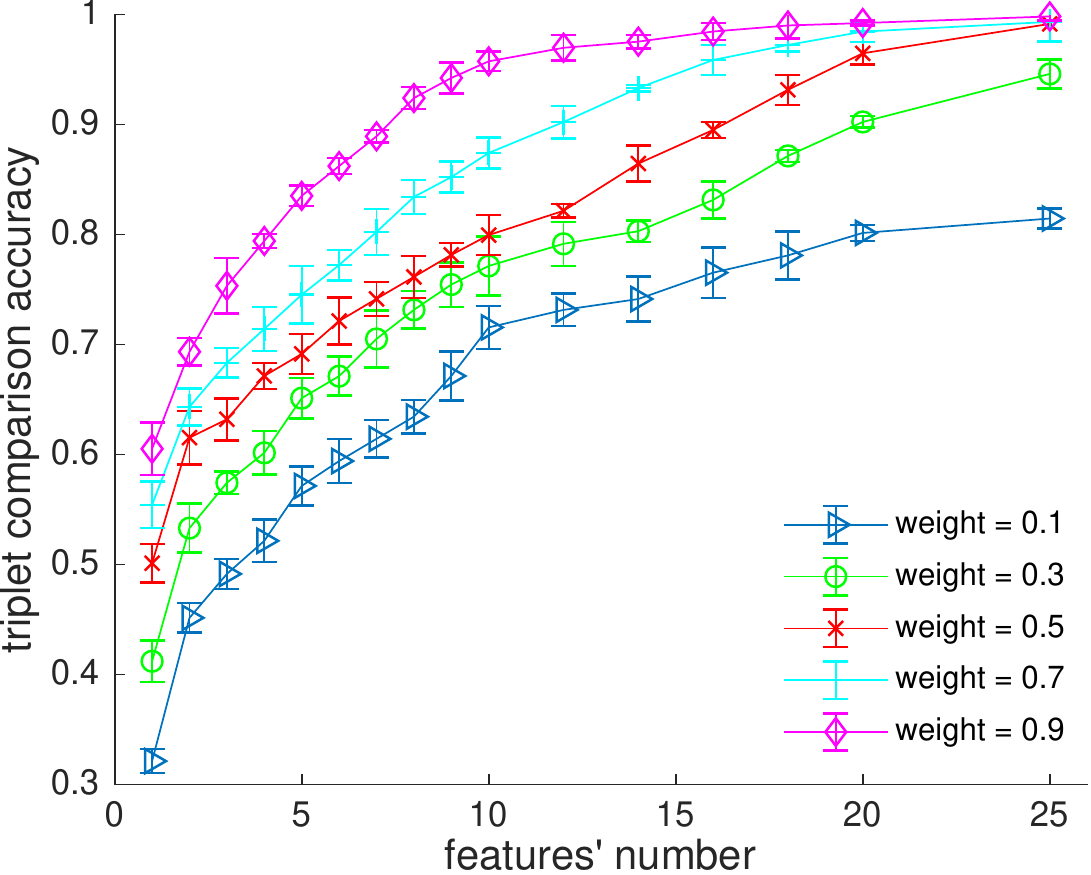}
\includegraphics[width =  0.4\textwidth]{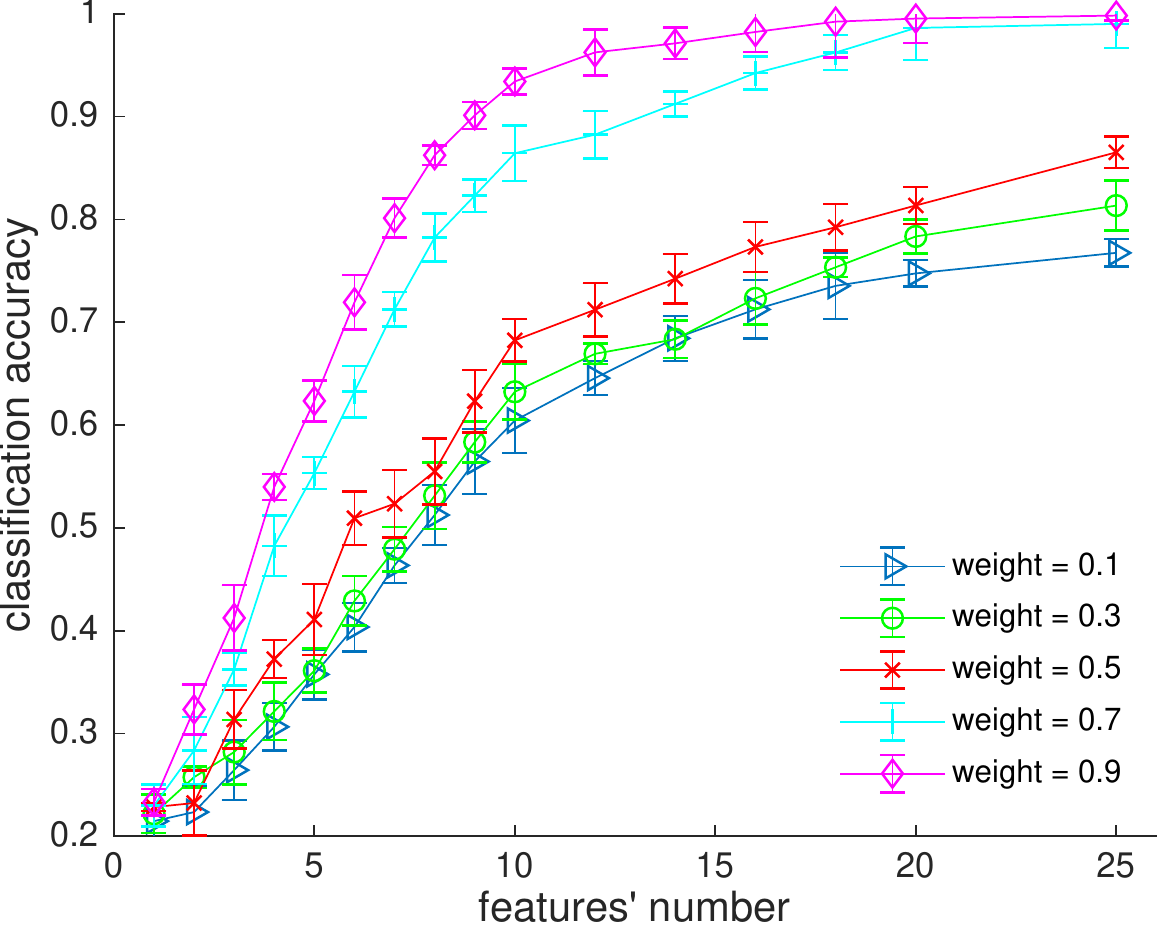}
{\caption{Synthetic data classification performance (with $20$ classes). (Top: CPML-s; bottom: CPML-m.)}}
\label{syn_result_class_200}
\end{figure*}

\subsection{Synthetic data}
The synthetic data is generated as follows. For each categorical feature, their values are generated by a multinomial distribution, which is parameterized by uniform random variables. To distinguish different classe, we manually add a weight to one of the component for each class. After the parameter normalization, the generated feature values can ensure favored values in each feature for different classes. Thus, a total of $D$ multinomial distributions are demanded for this generation.


More specifically, we set the size of the dataset as $1000$. The number of features varies in $[1, 2, 3, 4, 5, 6, 7, 8, 9, 10, 12, 14, 16, 18, 20, 25]$, and the weights are chosen sequentially in $[0.1, 0.3, 0.5, 0.7, 0.9]$.

\subsubsection{Impact of the Number of Features and of the Weights}
Fig.~\ref{syn_result} and Fig.~\ref{syn_result_class_200} shows the triplet comparison accuracy and pair-comparison classification accuracy for our methods of CPML-s and CPML-m. In these comparisons, we test the cases with $4$ classes and $20$ classes and use the trace-norm as the regularizer in this case, which is our Schatten $1$-norm.

From these figures, we can easily see that, in general, the performance (i.e. triplet comparison accuracy and pair-comparison classification accuracy) improves when the number of features increases. This phenomenon coincides with our common knowledge that the larger the number of informative features we have, the better performance will be. On the weight comparison, it is also clear that the performance will be better with larger weight values.

{These four figures also show that CPML-s and CPML-m have similar performance (we did not find any significant differences between the two models). In summary, the performance tends to be stable when the number of features lies between $10$ and $15$. In that case, one obtains satisfying results even on the least informative features (i.e. weight = 0.1). For the cases with $4$ classes and $20$ classes, we can see they have similar performance trends. The score of the latter looks to be a bit degraded, this might due to the larger number of classes. }

\subsubsection{Impact of Noisy Features}
We assess here whether the presence of noisy features impact the valid calculation of distance. We are using the CPML-s as the exploratory method, and the norms include the trace norm $Tr(M)$, the Frobenius Norm $Tr(M^2)$ and the Schatten $3$-norm $Tr(M^{3})$. $8$ informative features, with favored weight $0.3$, are used here with a set of noisy features (\# N.F. denotes the number of noisy features). We then compute the ratio between the norm of the metric on noisy feature and the norm of the metric on the whole feature space, i.e.  $\frac{r(M_{N.F.})}{r(M)}$. The results are shown in Table \ref{noisy_features}.

\begin{table}[htbp]
\caption{Noisy feature's ratio comparison (mean $\pm$ standard deviation)} \label{noisy_features}
\centering
\begin{tabular}{c|c|c|c}
\hline
\# N.F. & Trace & Frobenius & $L_3$ \\
\hline
1 & $0.006 \pm 0.006$ & $0.006 \pm 0.005$ & $0.004 \pm 0.003$ \\
4 & $0.012 \pm 0.006$ & $0.010 \pm 0.014$ & $0.017 \pm 0.015$ \\
7 & $0.011 \pm 0.006$ & $0.002 \pm 0.016$ & $0.017 \pm 0.025$ \\
10 & $0.002 \pm 0.006$ & $0.001 \pm 0.024$ & $0.005 \pm 0.014$ \\
13 & $0.017 \pm 0.006$ & $0.016 \pm 0.092$ & $0.006 \pm 0.070$ \\
 \hline
\end{tabular}%
\end{table}
From Table \ref{noisy_features}, we can see that, even if the number of noisy features is much larger than the number of meaningful features~($13$ noisy features and $8$ meaningful features), our CPML-s method with various Schatten $p$-norms can successfully control the noisy features as their influence on the metric does not exceed $1.7\%$. This noisy feature resistance property may be explained as the result of our simple distance definitions.

{\subsubsection{Running Time Comparison}
Fig. \ref{run_time} displays comparisons on the logarithm of the running time over different methods. The left part shows the performance of CPML-s, CPML-m and KDML in terms of different classes. As we can see, our CPML-s model is the fastest one  when compared to CPML-m and KDML. Also, the running time of CPML-s does not show a significant difference on the different choices for the number of classes. In contrast, both of CPML-m and KDML require heavier computation, and their running time depends on the number of classes. An interesting observation is that KDML is  faster than CPML-m when the number of classes is small. The right part shows the performance of CPML-s, CPML-m, KDML, LMNN, HELIC and POLA when the number of classes are set to $12$. Due to the online learning nature, POLA is the fastest to obtain the result as it only needs to scan the whole dataset once. Among all other comparison methods, CPML-s require the smallest running time. When HELIC requires smaller running time than the CPML-m algorithm, KDML and LMNN usually require more running time than the CPML-m algorithm (especially the size of dataset is large). }

\begin{figure*}[htbp]
\centering
\includegraphics[scale=0.5, width = 0.4 \textwidth]{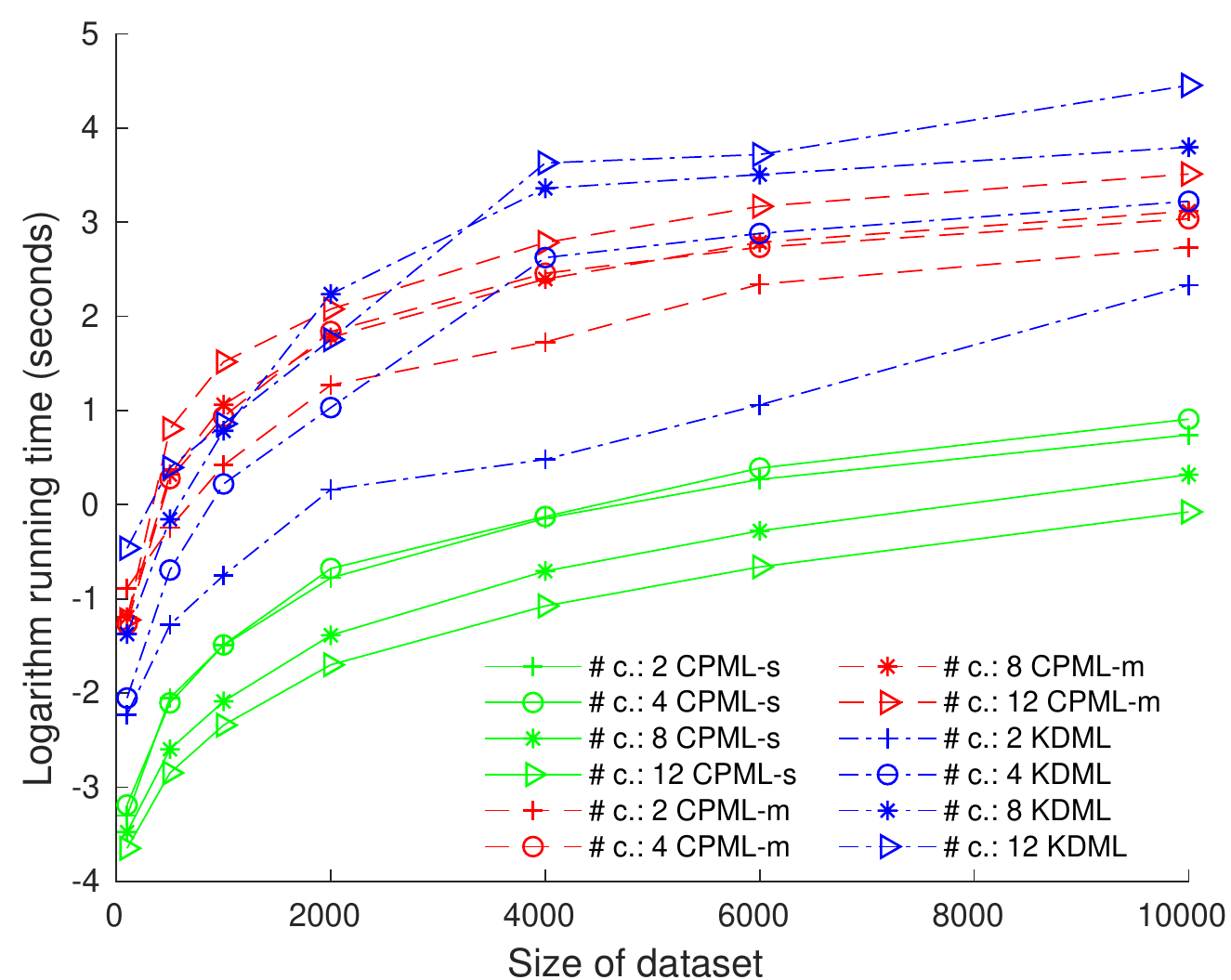}
\includegraphics[scale=0.5, width = 0.4 \textwidth]{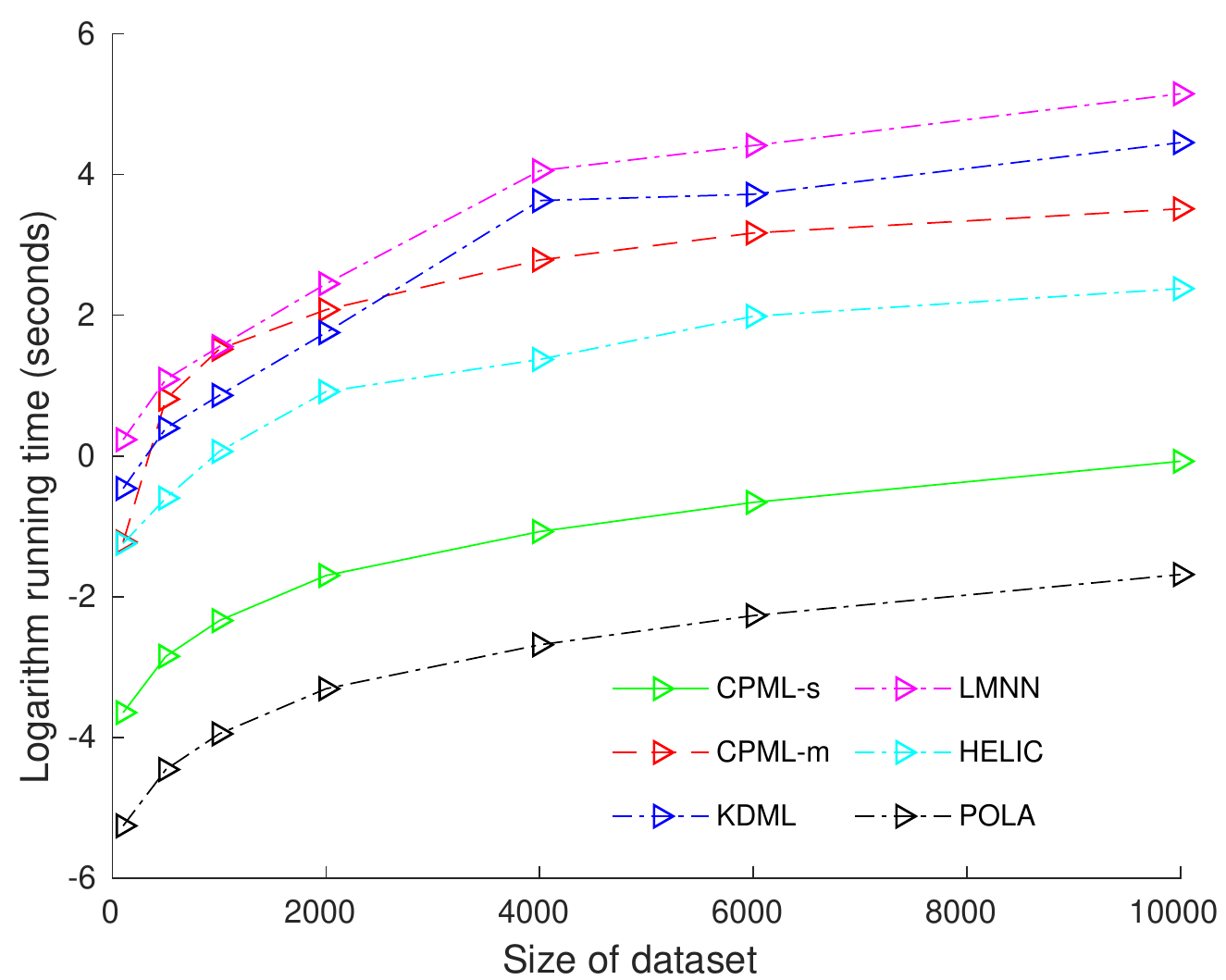}
\caption{Left: logarithm running time comparison of CPML-s, CPML-m and KDML with respect to different number of classes. Here \# c. denotes the number of classes. The blue lines denote the required running time for KDML, the red lines is for our CPML-m and the green lines CPML-s; right: logarithm running time comparison of CPML-s, CPML-m, KDML, LMNN, HELIC and POLA with respect to different sizes of datasets. }
\label{run_time}
\end{figure*}


{\subsection{Real world datasets \cite{Bache+Lichman:2013}}
We select $23$ real world datasets to test the performance of the CPM framework: Car, Balance, Mushroom, Voting, Nursery, Monks1, Monks2, Tic-tac-toe, Krkopt, Adult, Connect4, Census, Zoo, DNAPromoter, Lymphography, Audiology, Housevotes, Spect, Soybeanlarge, DNANomial, Splice, Krvskp and Led24. The detail information of these $23$ datasets, including number of instances (\# I.), number of categorical features (\# C.F.) and number of classes (\# C.), is shown in Table \ref{uci_inf}.

\begin{table}[htbp]
\caption{UCI datasets' detail information} \label{uci_inf}
\centering
\begin{tabular}{c|c|c|c}
\hline
Dataset & \# I. & \# C.F. & \# C. \\
\hline
Balance  & $625 $ & $ 4 $ &  $3$ \\
Car  & $1,728$ & $6$ & $4$ \\
Krkopt & $28,056 $ & $ 6 $ & $ 17 $\\
Monks2  & $12,958 $ & $ 8 $ & $ 4$\\
Tic-tac-toe  & $958 $ & $ 8 $ & $ 4$\\
Adult & $30,162 $ & $ 8 $ & $ 2 $\\
Nursery  & $12,960 $ & $ 8 $ & $ 4$\\
Zoo & $101 $ & $ 16 $ & $ 7 $\\
Housevotes & $232 $ & $ 16 $ & $ 2 $\\
Voting & $435 $ & $ 16 $ & $ 2$\\
Monks1 & $435 $ & $ 16 $ & $ 2$\\
Lymphography & $148 $ & $ 18 $ & $ 4 $\\
Mushroom  & $8,124 $ & $ 22 $ & $ 2 $\\
Spect & $267 $ & $ 22 $ & $ 3 $\\
Led24 & $3,200 $ & $ 24 $ & $ 10 $\\
Soybeanlarge & $307 $ & $ 35 $ & $ 19 $\\
Census & $299,285 $ & $ 35 $ & $ 2 $\\
Krvskp & $3,196 $ & $ 36 $ & $ 2 $\\
Connect4 & $67,557 $ & $ 42 $ & $ 3 $\\
DNAPromoter & $106 $ & $ 57 $ & $ 2 $\\
DNANominal & $3,186 $ & $ 60 $ & $ 4 $\\
Splice & $3,190 $ & $ 60 $ & $ 4 $\\
Audiology & $226 $ & $ 69 $ & $ 24 $\\
\hline
\end{tabular}
\end{table}

As shown by the results in Table~\ref{result_1} and Table~\ref{result_2}, our methods achieve very competitive performances against other baseline methods. In all the $23$ datasets, the CPML methods obtain the best performance in most cases. Even if in some datasets like Voting they may not be the best, their performance is quite close to it. What is more, their performance is consistent among different datasets. For other comparison methods, the HELIC and DM3 usually perform better than the KDML and LMNN. This might due to that CPML, HELIC and DM3 are specially designed for categorical data.

\begin{table*}[htbp] 
{\caption{Different models' performance on {\bf triplet comparison} accuracy (mean $\pm$ standard deviation)}}
\centering
\label{result_1}
\begin{tabular}{c|c|c|c|c|c|c|c}
\hline
Dataset & Euclidean & CPML-single & CPML-multi & KDML & LMNN & HELIC & DM3 \\
\hline
Car  & $0.593 \pm 0.020$  & ${0.619 \pm 0.023}$ & $\pmb{0.624 \pm 0.019}$   & $0.579 \pm 0.022$ &  $0.579 \pm 0.013$ &  $0.617 \pm 0.019$ &  $0.609 \pm 0.008$\\
Balance & $0.693 \pm 0.033$  & $0.848 \pm 0.025$ & $0.843 \pm 0.016$ & $0.851 \pm 0.023$ & $0.849 \pm 0.003$ &  $\pmb{0.853 \pm 0.017}$ &  $0.839 \pm 0.021$\\
Mushroom & $0.864 \pm 0.004$  & $\pmb{0.894 \pm 0.004}$  & $0.877 \pm 0.013$  & $0.879 \pm 0.004$ & $0.891 \pm 0.004$ &  $0.889 \pm 0.008$ &  $0.864 \pm 0.007$\\
Voting  & $0.894 \pm 0.041$  & $\pmb{0.934 \pm 0.040}$ & $0.929 \pm 0.031$ & $0.929 \pm 0.060$ & $0.926 \pm 0.003$ &  $0.909 \pm 0.010$ &  $0.910 \pm 0.007$ \\
Nursery & $0.897 \pm 0.004$  & $\pmb{0.901 \pm 0.005}$  & $0.898 \pm 0.004$  & $0.894 \pm 0.006$  & $0.890 \pm 0.006$ &  $0.878 \pm 0.007$ &  $0.893 \pm 0.010$ \\
Monks1 & $0.696 \pm 0.004$  & $0.776 \pm 0.004$  & $\pmb{0.781 \pm 0.013}$  & $0.780 \pm 0.004$ & $0.775 \pm 0.001$ &  $0.753 \pm 0.010$ &  $0.740 \pm 0.013$ \\
Monks2  & $0.556 \pm 0.032$  & $0.796 \pm 0.031$ & $\pmb{0.802 \pm 0.021}$ & $0.789 \pm 0.037$ & $0.793 \pm 0.034$ &  $0.770 \pm 0.010$ &  $0.749 \pm 0.008$ \\
Tic-tac-toe & $0.568 \pm 0.003$  & $0.581 \pm 0.006$  & $0.576 \pm 0.007$  & $0.580 \pm 0.008$  & $0.586 \pm 0.006$ &  $\pmb{0.587 \pm 0.011}$ &  $0.563 \pm 0.009$ \\
Krkopt  & $0.523 \pm 0.004$  & $0.546 \pm 0.023$ & $0.546 \pm 0.012$   & $0.531 \pm 0.019$ &  $0.529 \pm 0.003$ &  $0.537 \pm 0.010$ &  $\pmb{0.549 \pm 0.007}$ \\
Adult & $0.791 \pm 0.010$  & $0.852 \pm 0.020$ & $\pmb{0.861 \pm 0.008}$ & $0.832 \pm 0.007$ & $0.828 \pm 0.005$ &  $0.849 \pm 0.009$ &  $0.837 \pm 0.007$ \\
Connect4 & $0.634 \pm 0.009$  & $\pmb{0.681 \pm 0.012}$  & $0.676 \pm 0.011$  & $0.663 \pm 0.010$ & $0.659 \pm 0.014$ &  $0.656 \pm 0.009$ &  $0.673 \pm 0.019$ \\
Census  & $0.751 \pm 0.013$  & $0.810 \pm 0.007$ & $\pmb{0.816 \pm 0.013}$ & $0.791 \pm 0.016$ & $0.742 \pm 0.024$ &  $0.731 \pm 0.006$ &  $0.748 \pm 0.003$ \\
Zoo & $\pmb{1.000 \pm 0.000}$  & $\pmb{1.000 \pm 0.000}$  & $\pmb{1.000 \pm 0.000}$  & $\pmb{1.000 \pm 0.000}$  & $\pmb{1.000 \pm 0.000}$ &  $\pmb{1.000 \pm 0.000}$ &  $\pmb{1.000 \pm 0.000}$ \\
DNAPromoter & $0.949 \pm 0.014$  & $\pmb{0.973 \pm 0.016}$  & $0.969 \pm 0.011$  & $0.931 \pm 0.015$ & $0.934 \pm 0.013$ &  $0.956 \pm 0.022$ &  $0.946 \pm 0.012$ \\
Lymphography  & $0.861 \pm 0.011$  & $0.882 \pm 0.010$ & $\pmb{0.885 \pm 0.010}$ & $0.867 \pm 0.007$ & $0.873 \pm 0.012$ &  $0.871 \pm 0.009$ &  $0.869 \pm 0.006$ \\
Audiology & $0.742 \pm 0.017$  & $0.773 \pm 0.018$  & $\pmb{0.784 \pm 0.015}$  & $0.739 \pm 0.017$  & $0.759 \pm 0.011$ &  $0.769 \pm 0.009$ &  $0.736 \pm 0.015$ \\
Housevotes  & $0.945 \pm 0.010$  & $0.975 \pm 0.006$ & $\pmb{0.983 \pm 0.005}$   & $0.952 \pm 0.005$ &  $0.962 \pm 0.017$ &  $0.977 \pm 0.016$ &  $0.978 \pm 0.009$ \\
Spect & $0.603 \pm 0.011$  & $0.642 \pm 0.014$ & $\pmb{0.659 \pm 0.029}$ & $0.598 \pm 0.011$ & $0.630 \pm 0.003$ &  $0.626 \pm 0.018$ &  $0.637 \pm 0.014$ \\
Soybeanlarge & $0.803 \pm 0.005$  & $0.846 \pm 0.014$  & $\pmb{0.853 \pm 0.014}$  & $0.827 \pm 0.009$ & $0.836 \pm 0.011$ &  $0.818 \pm 0.008$ &  $0.822 \pm 0.011$ \\
DNANominal  & $0.921 \pm 0.008$  & $\pmb{0.953 \pm 0.017}$ & $0.949 \pm 0.009$ & $0.932 \pm 0.005$ & $0.941 \pm 0.017$ &  $0.924 \pm 0.008$ &  $0.935 \pm 0.017$ \\
Splice & $0.860 \pm 0.010$  & $\pmb{0.898 \pm 0.012}$  & $0.894 \pm 0.010$  & $0.885 \pm 0.011$  & $0.880 \pm 0.009$ &  $0.868 \pm 0.004$ &  $0.877 \pm 0.013$ \\
Krvskp & $0.950 \pm 0.013$  & $0.975 \pm 0.001$  & $\pmb{0.976 \pm 0.002}$  & $0.959 \pm 0.005$ & $0.967 \pm 0.003$ &  $0.924 \pm 0.013$ &  $0.954 \pm 0.003$ \\
Led24  & $0.642 \pm 0.007$  & $\pmb{0.713 \pm 0.005}$ & $0.689 \pm 0.018$ & $0.676 \pm 0.013$ & $0.662 \pm 0.012$ &  $0.678 \pm 0.014$ &  $0.655 \pm 0.011$ \\
\hline
\end{tabular}%
\end{table*}

\begin{table*}[t] 
{\caption{Different models' performance on {\bf classification} accuracy (mean $\pm$ standard deviation)}}
\centering
\label{result_2}
\begin{tabular}{c|c|c|c|c|c|c|c}
\hline
Dataset  & Euclidean & CPML-single & CPML-multi & KDML & LMNN & HELIC & DM3\\
\hline
Car  & $0.969 \pm 0.013$  & $0.970 \pm 0.012$  & $\pmb{0.977 \pm 0.010}$ & $\pmb{0.977 \pm 0.012}$  & $0.975 \pm 0.003$ & $0.969 \pm 0.005$ & $0.966 \pm 0.004$ \\
Balance & $0.866 \pm 0.031$  & $0.938 \pm 0.030$  & $0.929 \pm 0.023$  & $0.936 \pm 0.031$ & $0.940 \pm 0.005$ &  $\pmb{0.941 \pm 0.006}$ &  $0.937 \pm 0.007$ \\
Mushroom & $\pmb{1.000 \pm 0.000}$  & $\pmb{1.000 \pm 0.000}$  & $\pmb{1.000 \pm 0.000}$ & $\pmb{1.000 \pm 0.000}$  & $\pmb{1.000 \pm 0.000}$ &  $\pmb{1.000 \pm 0.000}$ &  $\pmb{1.000 \pm 0.000}$ \\
Voting  & $0.938 \pm 0.032$  & $0.943 \pm 0.043$ & $0.935 \pm 0.037$ & $0.915 \pm 0.112$ & $0.923 \pm 0.007$ &  $0.939 \pm 0.007$ &  $0.924 \pm 0.007$ \\
Nursery  & $0.983 \pm 0.003$  & $0.988 \pm 0.003$   & $\pmb{0.994 \pm 0.003}$   & $0.989 \pm 0.003$  & $0.985 \pm 0.010$ &  $0.993 \pm 0.008$ &  $0.984 \pm 0.006$ \\
Monks1 & $0.842 \pm 0.006$  & $0.849 \pm 0.005$  & $\pmb{0.859 \pm 0.010}$ & $0.856 \pm 0.007$  & $0.853 \pm 0.003$ &  $0.839 \pm 0.017$ &  $0.845 \pm 0.009$ \\
Monks2 & $0.711 \pm 0.001$  & $0.721 \pm 0.002$ & $\pmb{0.723 \pm 0.002}$ & $0.719 \pm 0.006$ & $0.717 \pm 0.003$ &  $0.709 \pm 0.005$ &  $0.713 \pm 0.005$ \\
Tic-tac-toe  & $0.862 \pm 0.006$  & $0.874 \pm 0.001$   & $0.860 \pm 0.007$   & $0.880 \pm 0.010$  & $0.876 \pm 0.008$ &  $\pmb{0.883 \pm 0.007}$ &  $0.864 \pm 0.010$ \\
Krkopt  & $0.513 \pm 0.012$  & $0.519 \pm 0.019$ & $0.523 \pm 0.007$   & $0.508 \pm 0.017$ &  $0.514 \pm 0.005$ &  $\pmb{0.548 \pm 0.007}$ &  $0.537 \pm 0.009$ \\
Adult & $0.837 \pm 0.011$  & $0.853 \pm 0.019$ & $\pmb{0.854 \pm 0.010}$ & $0.839 \pm 0.011$ & $0.836 \pm 0.006$ &  $0.845 \pm 0.007$ &  $0.823 \pm 0.008$ \\
Connect4 & $0.561 \pm 0.007$  & $\pmb{0.571 \pm 0.016}$  & $0.569 \pm 0.009$  & $0.564 \pm 0.007$ & $0.568 \pm 0.012$ &  $0.546 \pm 0.007$ &  $0.556 \pm 0.023$ \\
Census  & $0.671 \pm 0.007$  & $\pmb{0.687 \pm 0.009}$ & $0.685 \pm 0.009$ & $0.681 \pm 0.012$ & $0.675 \pm 0.009$ &  $0.667 \pm 0.007$ &  $0.677 \pm 0.016$ \\
Zoo & $\pmb{1.000 \pm 0.000}$  & $\pmb{1.000 \pm 0.000}$  & $\pmb{1.000 \pm 0.000}$ & $\pmb{1.000 \pm 0.000}$  & $\pmb{1.000 \pm 0.000}$ &  $\pmb{1.000 \pm 0.000}$ &  $\pmb{1.000 \pm 0.000}$ \\
DNAPromoter & $0.896 \pm 0.010$  & $\pmb{0.919 \pm 0.007}$  & $0.916 \pm 0.006$  & $0.903 \pm 0.008$ & $0.902 \pm 0.010$ &  $0.907 \pm 0.009$ &  $0.904 \pm 0.012$ \\
Lymphography  & $0.857 \pm 0.005$  & $\pmb{0.879 \pm 0.008}$ & $0.875 \pm 0.010$ & $0.865 \pm 0.005$ & $0.871 \pm 0.009$ &  $0.866 \pm 0.009$ &  $0.855 \pm 0.003$ \\
Audiology & $0.682 \pm 0.009$  & $0.701 \pm 0.009$  & $\pmb{0.713 \pm 0.009}$  & $0.671 \pm 0.011$  & $0.647 \pm 0.009$ &  $0.708 \pm 0.009$ &  $0.710 \pm 0.018$ \\
Housevotes  & $0.906 \pm 0.015$  & $\pmb{0.941 \pm 0.011}$ & $0.932 \pm 0.018$   & $0.925 \pm 0.017$ &  $0.921 \pm 0.020$ &  $0.922 \pm 0.008$ &  $0.931 \pm 0.014$ \\
Spect & $0.531 \pm 0.007$  & $0.549 \pm 0.004$ & $0.552 \pm 0.009$ & $0.538 \pm 0.015$ & $0.537 \pm 0.005$ &  $0.542 \pm 0.005$ &  $\pmb{0.558 \pm 0.013}$ \\
Soybeanlarge & $0.859 \pm 0.018$  & $0.879 \pm 0.009$  & $\pmb{0.885 \pm 0.004}$  & $0.869 \pm 0.007$ & $0.871 \pm 0.009$ &  $0.869 \pm 0.006$ &  $0.870 \pm 0.015$ \\
DNANominal  & $0.871 \pm 0.032$  & $0.928 \pm 0.021$ & $\pmb{0.936 \pm 0.018}$ & $0.908 \pm 0.015$ & $0.924 \pm 0.007$ &  $0.919 \pm 0.007$ &  $0.909 \pm 0.007$ \\
Splice & $0.834 \pm 0.009$  & $0.879 \pm 0.007$  & $\pmb{0.881 \pm 0.012}$  & $0.851 \pm 0.009$  & $0.862 \pm 0.007$ &  $0.868 \pm 0.006$ &  $0.846 \pm 0.022$ \\
Krvskp & $0.930 \pm 0.010$  & $0.945 \pm 0.003$  & $\pmb{0.946 \pm 0.008}$  & $0.918 \pm 0.035$ & $0.924 \pm 0.009$ &  $0.923 \pm 0.007$ &  $0.915 \pm 0.011$ \\
Led24  & $0.597 \pm 0.019$  & $0.642 \pm 0.007$ & $\pmb{0.651 \pm 0.031}$ & $0.601 \pm 0.008$ & $0.621 \pm 0.019$ &  $0.648 \pm 0.008$ &  $0.643 \pm 0.007$ \\
\hline
\end{tabular}%
\end{table*}
}


\section{Conclusion} \label{sec_7}
Through the experimental validation, we have seen that our models provide competitive and robust classification results compared to previous ones. At the same time, the required computation time is reduced significantly, especially for our CPML-s method.

There are many future work that can be done based on this framework. The number of values for each feature, as well as the distribution of these values, influence the distance calculation and this should be considered. Moreover, our learning task fits within the supervised learning scenario, where all the data points label are known before hand. In practice, however, one sometimes has access to only side-information concerning labels, which corresponds to a semi-supervised learning scenario. A new learning design is needed for this scenario, that requires defining a new projection function.

\bibliographystyle{IEEEtran}
\bibliography{SDM2015}

\appendices

{\section{Proof of Property 1}
For the categorical projected multi distance~(CPm), we have
\begin{align}
& d_M(\boldsymbol{x}_i, \boldsymbol{x}_j) \nonumber \\
=&\sum_{c=1}^C\sum_{p, q}A^{ij}_{c,pq}M_{c,pq} \nonumber \\
= & \sum_{c=1}^C (\phi_c(\boldsymbol{x}_i)-\phi_c(\boldsymbol{x}_j))^{\top}M_c(\phi_c(\boldsymbol{x}_i)-\phi_c(\boldsymbol{x}_j))\ge 0
\end{align}

For the triangle-inequality, we first docompose the positive semi-definite matrix as $M_c = Q^T\Lambda Q$, where $Q$ is the orthogonal matrix and $\Lambda$ is the diagonal matrix as $\Lambda = \text{diag}(\lambda_1, \ldots, \lambda_D)$ and $\lambda_1\ge 0, \ldots, \lambda_D\ge 0$. Further, we let matrix $U$ to be defined as 
$U = \text{diag}(\sqrt{\lambda_1}, \ldots, \sqrt{\lambda_D})Q$. Thus, we have $M_c = U^{\top}U$. Let $\widehat{\phi}(\boldsymbol{x}_i) = U{\phi}(\boldsymbol{x}_i)$, it is easily to see that
\begin{align} \label{euc_eq}
&(\phi_c(\boldsymbol{x}_i)-\phi_c(\boldsymbol{x}_j))^{\top}M_c(\phi_c(\boldsymbol{x}_i)-\phi_c(\boldsymbol{x}_j)) \nonumber \\
= &(\widehat{\phi}_c(\boldsymbol{x}_i)-\widehat{\phi}_c(\boldsymbol{x}_j))^{\top}(\widehat{\phi}_c(\boldsymbol{x}_i)-\widehat{\phi}_c(\boldsymbol{x}_j))\nonumber \\
= & \sum_d (\widehat{\phi}_{cd}(\boldsymbol{x}_i)-\widehat{\phi}_{cd}(\boldsymbol{x}_j))^2
\end{align}
where $\widehat{\phi}_{cd}(\boldsymbol{x}_i)$ refers to the $d$-th element of $\widehat{\phi}_{c}(\boldsymbol{x}_i)$. Eq.~(\ref{euc_eq}) shows the distance can be alternatively represented as a Euclidean distance format. Based on the triangle-inequality of the Euclidean distance, we can straightforwardly get the triangle-inequality of the categorical projected multi-distance. 

For the categorical projected single distance~(CPs), we have
\begin{align}
& d_M(\boldsymbol{x}_i, \boldsymbol{x}_j) \nonumber \\
=&C\sum_{p, q}A^{ij}_{pq}M_{pq} \nonumber \\
= & \sum_{c=1}^C (\phi_c(\boldsymbol{x}_i)-\phi_c(\boldsymbol{x}_j))^{\top}M(\phi_c(\boldsymbol{x}_i)-\phi_c(\boldsymbol{x}_j))\ge 0
\end{align}
The triangle inequality of the categorical projected single distance~(CPs) can be obtained in a similar way. }

\end{document}